 \documentclass[final,5p,times,twocolumn]{elsarticle}

\usepackage{graphicx}

\usepackage{amssymb}
\usepackage{amsthm}

\usepackage{lineno}

\usepackage{amsmath}
\usepackage{epstopdf}
\usepackage{subfigure}
\usepackage{multirow}
\usepackage{algorithm}
\usepackage{algorithmic}
\usepackage{hyperref}

\newtheorem{lemma}{Lemma}

\biboptions{square, numbers, sort, compress}

\journal{Knowledge-Based Systems}

\begin{document}

\begin{frontmatter}



\title{Subspace Clustering Using a Symmetric Low-Rank Representation}
\tnotetext[t1]{This work was supported by the National Science Foundation of China (Grant No. 61432012, U1435213, 61402306 and 61602329).}

\author{Jie Chen}
\author{Hua Mao}
\author{Yongsheng Sang}
\author{Zhang Yi \corref{cor1}}
\ead{zyiscu@gmail.com}

\cortext[cor1]{Corresponding author}

\address{Machine Intelligence Laboratory, College of Computer Science, Sichuan University, Chengdu 610065, P.R. China}

\begin{abstract}

In this paper, we propose a low-rank representation with symmetric constraint (LRRSC) method for robust subspace clustering. Given a collection of data points approximately drawn from multiple subspaces, the proposed technique can simultaneously recover the dimension and members of each subspace. LRRSC extends the original low-rank representation algorithm by integrating a symmetric constraint into the low-rankness property of high-dimensional data representation. The symmetric low-rank representation, which preserves the subspace structures of high-dimensional data, guarantees weight consistency for each pair of data points so that highly correlated data points of subspaces are represented together. Moreover, it can be efficiently calculated by solving a convex optimization problem. We provide a proof for minimizing the nuclear-norm regularized least square problem with a symmetric constraint. The affinity matrix for spectral clustering can be obtained by further exploiting the angular information of the principal directions of the symmetric low-rank representation. This is a critical step towards evaluating the memberships between data points. Besides, we also develop eLRRSC algorithm to improve the scalability of the original LRRSC by considering its closed form solution. Experimental results on benchmark databases demonstrate the effectiveness and robustness of LRRSC and its variant compared with several state-of-the-art subspace clustering algorithms.

\end{abstract}


\end{frontmatter}


\section{Introduction}
\label{Introduction}

In recent years, subspace clustering techniques have attracted much attention from researchers in many areas; for example, computer vision, machine learning, and pattern recognition. Subspace clustering is important to numerous applications such as image representation \cite{Hong2006ImgRr, Liu2010LRR1}, clustering \cite{Ho2003Cluster, Liu2010LRR, Elhamifar2013SSC, Favarob2011LRSC, Vidala2013LRSC}, and motion segmentation \cite{Yang2006MotionSeg, Rao2010MotionSeg, Zappella2011, Pham2012ISC, Zhuang2012NLRR}. The generality and importance of subspaces naturally lead to the challenging problem of subspace clustering, where the goal is to simultaneously segment data into clusters that correspond to a low-dimensional subspace. To be more specific, given a set of data points drawn from a mixture of subspaces, the task is to segment all data points into their respective subspaces.

When we consider subspace clustering in real applications, there is a large amount of high-dimensional data available; for example, digital images, video surveillance, and traffic monitoring. High-dimension data increase the computational cost of algorithms and have a negative effect on performance because of noise and corrupted observations. It is typical to impose an assumption that the high-dimensional data are approximately drawn from a union of multiple subspaces. This assumption is reasonable because data in a class can be well represented by a low-dimensional subspace of the high-dimensional ambient space. In fact, high-dimensional data often have a smaller intrinsic dimension. Examples of high-dimensional data that lie in a low-dimensional subspace of the ambient space include images of an individual's face captured under various laboratory-controlled lighting conditions, handwritten images of a digit with different rotations, translations, and thicknesses, and feature trajectories of a moving object in a video \cite{Basri2003, Boult1991Factor, Qu007SP}. This has motivated the development of a number of techniques \cite{Tibshiran1996Lasso, Donoho2006MinimalL1Norm, Wright2009RPCA, Candes2010MC, Lin2011ALM} for finding and exploiting low-dimensional structures in high-dimensional data.

A number of methods have been devised to exactly solve the subspace clustering problem \cite{Vidal2011SC}. Subspace clustering methods can be roughly divided into statistical learning based \cite{Fischler1981RANSAC}, factorization based \cite{Costeira1998MF, Govindu2005FC}, algebra based \cite{Vidal2005GPCA}, iterative \cite{Ho2003KSC, Zhang2009AIS, ZYI2010RNN}, and spectral-type based methods \cite{Elhamifar2013SSC, Liu2010LRR, Ni2010LRRPSD, Zhuang2012NLRR, Lauer2009SC}. These methods can produce good results under the assumption that data are strictly drawn from linearly independent subspaces. However, an increase in difficulty is in part caused by the data not strictly following subspace structures (because of noise and corruptions), and can lead to indistinguishable subspace clustering. Therefore, subspace clustering algorithms that take into account the multiple subspace structure of high-dimensional data are required.

Recently, some work on the Frobenius norm, sparse representation theory and rank minimization \cite{Cand2008l1norm, Recht2010MinRank, Cai2010SVT, Elhamifar2013SSC, Liu2010LRR, Ni2010LRRPSD, Toh2010NNWithLS, Liu2016DRSC, Lu2013L2Graph, Peng2016SL, Peng2016SC, Du2017SC} have recently been proposed to alleviate some of aforementioned drawbacks. Effective approaches to subspace clustering called spectral-type based methods have been developed. These methods typically perform subspace clustering in two stages: first learn an affinity matrix (an undirected graph) from the given data, and then obtain the final clustering results using a spectral clustering algorithm \cite{Luxburg2007SC} such as normalized cuts (NCuts) \cite{Shi2000Ncuts}. These techniques can effectively recover the multiple subspace structures when high-dimensional data is grossly corrupt. However, there are still several open questions, e.g., how to choose a good affinity matrix by building a similarity graph, and how to estimate the number and dimensions of underlying subspaces in a time efficient manner. For a given data matrix $X = [{x_1},{x_2}...,{x_N}] \in {\mathbf{R}^{d \times N}}$, each of which can be represented by the combination of the basis in the dictionary $A = [{a_1},{a_2},...,{a_n}] \in {\mathbf{R}^{d \times n}}$:
\begin{equation}\label{eq:lrr}
X=AZ,
\end{equation}
where it refers to $Z$ as a coefficient matrix of linear representation. Some techniques based on Frobenius norm are developed to construct L2-Graph for subspace clustering, which require low computational cost \cite{Lu2013L2Graph, Peng2015SC, Peng2016SL}. Besides, \citet{Elhamifar2013SSC} proposed a sparse subspace clustering (SSC) method, which uses the sparsest representation of the data points produced by ${l_1}$-norm minimization of the coefficient matrix of linear representation to define an affinity matrix of the undirected graph. \citet{Liu2016DRSC} proposed an efficient distributed framework for the computation of SSC on a shared-memory architecture. Then spectral clustering techniques are used to perform subspace clustering. The convex optimization model of SSC, under the assumption that the subspaces are either linearly independent or disjoint under certain conditions, results in a block diagonal solution. However, there is no global structural constraint on the spare representation. Moreover, SSC needs to carefully tune the number of nearest neighbors to improve performance, which leads to another parameter.

\citet{Liu2010LRR} recently introduced low-rank representation techniques into the subspace clustering problem and established a low-rank representation (LRR) algorithm \cite{Liu2010LRR}. LRR assumes that the data samples are approximately drawn from a mixture of multiple low-rank subspaces. The goal of LRR is to take the correlation structure of data into account, and find the lowest-rank representation of $Z$ using an appropriate dictionary. LRR solves the convex optimization problem of nuclear-norm minimization, which is considered as a surrogate for rank minimization. It performs subspace clustering excellently. However, the affinity for the spectral clustering input, which can be computed using a symmetrization step of the low-rank representation results \cite{Liu2010LRR1}, is not good at characterizing how other samples contribute to the reconstruction of a given sample. In addition, the motivation of a new version of LRR to use the matrix ${U^*}{U^{*T}}$ as an affinity for the spectral clustering input remains vague without theoretical analysis \cite{Liu2010LRR}. Here, ${U^*}$ can found using the skinny SVD of the low-rank representation $Z$, i.e., $Z = {U^*}{\Sigma ^*}{V^{*T}}$. \citet{Ni2010LRRPSD} proposed a robust low-rank subspace segmentation method to extend LRR and improve performance. It enforces the symmetric positive semi-definite (PSD) constraint on $Z$ to explicitly obtain a symmetric PSD matrix and avoid the symmetrization post-processing. LRR and its variations are guaranteed to produce block-diagonal affinity matrices if the points are already ordered according to their respective clusters under the independence assumption. However, nonnegativity of the values of the PSD matrix cannot be guaranteed by PSD conditions. Consequently, negative elements of the PSD matrix lack physical interpretation for visual data. Besides, if the low-rank representation matrix is considered as a pairwise affinity relationship for spectral clustering between data points, this inevitably leads to loss of information. This means that it does not fully capture the complexity of the problem, i.e., the intrinsic correlation of data points \cite{Agarwal2005Clustering, Zhou2006HyperGraph}. To tackle this difficulty, one needs to learn an affinity matrix by exploiting the structure of the low-rank representation.

In this paper,  we present a low-rank representation with symmetric constraint (LRRSC) method and its variant (eLRRSC) for robust subspace clustering. In particular, our motivation is to integrate the symmetric constraint into the low-rankness property of high-dimensional data representation, to learn a symmetric low-rank matrix that preserves the subspace structures of high-dimensional data. By solving the nuclear-norm minimization problem of $Z$ in a simple and efficient way, we can learn a symmetric low-rank matrix. For example, given a set of data points, we represent each point as a linear combination of the others, where the low-rank coefficients should be symmetric. Low-rank representation techniques often suffer from heavy computational cost when require iterative SVD operations. In contrast with these techniques, eLRRSC can obtain a symmetric low-rank representation in a closed form solution, which dramatically reduces the computational complexity. To obtain the closed form solution, we provide a proof to minimize the nuclear-norm regularized least square problem with a symmetric constraint. Consequently, eLRRSC can be adopted by large-scale subspace clustering problems because of its advantages of computational stability and efficiency. Besides, the symmetric effect of LRRSC and eLRRSC guarantees weight consistency for each pair of data points, so that highly correlated data points are represented together. As mentioned above, using a symmetric matrix as the input for subspace clustering may negatively affect the performance. To overcome this drawback, it is critical to investigate the intrinsically geometrical structure of the memberships of data points preserved in a symmetric low-rank representation, i.e., the angular information of the principal directions of the symmetric low-rank representation. This can improve the subspace clustering performance. An affinity that encodes the memberships of subspaces can be constructed using the angular information of the normalized rows of  ${U^*}$, or columns of $V^{*T}$, obtained from the skinny SVD of the symmetric matrix $Z$, i.e., $Z = {U^*}{\Sigma ^*}{V^{*T}}$. With the learned affinity matrix, spectral clustering can segment the data into clusters with the underlying subspaces they are drawn form. The proposed algorithm not only recovers the dimensions of each subspace, but also effectively learns a symmetric low-rank representation for the purpose of subspace clustering. In contrast to LRR, LRRSC and eLRRSC can obtain a coefficient matrix symmetric, and then they builds a desired affinity for subspace clustering. Further details will be discussed in Section \ref{sec:LRRSC}.

The contributions of the paper are summarized as follows:
\begin{enumerate}[a)]
\item It incorporates the symmetry idea into low-rank representation learning, and can successfully learn a symmetric low-rank matrix for high-dimensional data representation. We have provided a proof for minimizing the nuclear-norm regularized least square problem with a symmetric constraint.
\item A symmetric low-rank representation can be obtained by eLRRSC in a closed form solution, which can be solved very efficiently using SVD techniques.
\item It exploits the intrinsically geometrical structure of the memberships of data points preserved in a symmetric low-rank matrix (i.e., the angular information of principal directions of the symmetric low-rank representation) to construct an affinity matrix with more separation ability. This significantly improves the subspace clustering performance.
\item Compared with other state-of-the-art methods, our extensive experimental results using benchmark databases demonstrate the effectiveness and robustness of LRRSC and eLRRSC for subspace clustering.
\end{enumerate}

The remainder of this paper is organized as follows. Section \ref{sec:Relatedwork} summarizes some related work on low-rank representation techniques that inspired this work. Section \ref{sec:LRRSC} presents the proposed LRRSC for subspace clustering. Extensive experimental results using benchmark databases are presented in Section \ref{sec:Experiments}. Finally, Section \ref{sec:Conclusions} concludes this paper.

\section{A review of previous work}
\label{sec:Relatedwork}

Consider a set of data vectors $X = [{x_1},{x_2}...,{x_n}] \in {\mathbf{R}^{d \times n}}$, each column of which is drawn from a union of $k$ subspaces $\{ {S_i}\} _{i = 1}^k$ with unknown dimensions. The goal of subspace clustering is to cluster data points into their respective subspaces. A main challenge in applying spectral clustering to subspace clustering is to define a good affinity matrix, each entry of which measures the similarity between data points ${x_i}$ and ${x_j}$. This section provides a review of low-rank representation techniques for solving subspace clustering problems that are closely related to the proposed method.

\subsection{Subspace clustering by low-rank representation}
\label{sec:ReviewLRR}

Recently, \citet{Liu2010LRR} proposed a novel objective function named the low-rank representation method for subspace clustering. Instead of seeking a sparse representation as in SSC, LRR seeks the lowest-rank representation among all the candidates that can represent the data samples as linear combinations of the bases in a given dictionary. SSC enforces that $Z$ is sparse by imposing an ${l_1}$-norm regularization on $Z$, while LRR encourages $Z$ to be low-rank using nuclear-norm regularization. By using the nuclear norm as a good surrogate for the rank function, LRR solves the following nuclear norm minimization problem for the noise free case:

\begin{equation}\label{eq:lrr}
\mathop {\min }\limits_{Z} {\left\| Z \right\|_*} \qquad s.t.\qquad X = AZ,
\end{equation}
where $A = [{a_1},{a_2},...,{a_n}] \in {\mathbf{R}^{d \times n}}$ is an overcomplete dictionary, and ${\left\| \cdot \right\|_*}$ denotes the nuclear norm (the sum of the singular values of the matrix). When the subspaces are independent, LRR succeeds in recovering the desired low-rank representations.

In the case of data being grossly corrupted by noise or outliers, the LRR algorithm solves the following convex optimization problem:

\begin{equation}\label{eq:lrrwithnoisy}
\mathop {\min }\limits_{Z,E} {\left\| Z \right\|_*} + \lambda {\left\| E \right\|_l}\qquad s.t.\qquad X = AZ + E,
\end{equation}
where \({\left\|  \cdot  \right\|_l}\) indicates a certain regularization strategy for characterizing various corruptions. For example, the ${l_{2,1}}$-norm encourages the columns of $E$ to be zero. By choosing an appropriate dictionary $A$, LRR seeks the lowest-rank representation matrix of the coefficient matrix $Z$. This can also be used to recover the clean data from the original samples. To reduce the computational complexity, LRR uses $XP^*$ as its dictionary, where $P^*$ can be computed by orthogonalizing the columns of $X^T$.

The above optimization problem is convex, and can be efficiently solved by the inexact augmented Lagrange multipliers (ALM) technique in polynomial time with a guaranteed high performance \cite{Lin2011ALM}. After obtaining an optimal solution \(({Z^ * },{E^ * })\), \({Z^ * }\) is used to define an affinity matrix $\left| Z \right| + {\left| Z \right|^T}$ for spectral clustering.

To avoid symmetrization post-processing, \citet{Ni2010LRRPSD} presented an improved LRR model with PSD constraints, which can be formulated as
\begin{equation}
\mathop {\min }\limits_{Z,E} {\left\| Z \right\|_*} + \lambda {\left\| E \right\|_l} \qquad s.t. \qquad X = XZ + E,Z \succeq 0.
\end{equation}

Rigorous mathematical derivations \cite{Ni2010LRRPSD} show that the LRR-PSD model is the perfect case of the LRR scheme, and establishes the uniqueness of the optimal solution. By applying a scheme similar to LRR, LRR-PSD can also be efficiently solved. Then, ${Z^ * }$ is used to define an affinity matrix $\left| Z \right| $, after acquiring an optimal solution $({Z^ * },{E^ * })$. To reduce computational cost in the LRR scheme, \citet{Chen2016SLRR} introduced the symmetric low-rank representation(SLRR) method to avoid iterative SVD operation. This significantly decrease the computational cost for the subspace clustering.

\section{Exploitation of a low-rank representation with a symmetric constraint}
\label{sec:LRRSC}

In this section, we propose a low-rank representation with a symmetric constraint (LRRSC). Our approach takes into consideration the intrinsically geometrical structure of the symmetric low-rank representation of high-dimensional data. We first propose a low-rank representation model with a symmetric constraint, which can be efficiently calculated by solving a convex optimization problem. Then, we learn an affinity matrix for subspace clustering by exploiting the intrinsically geometrical structure of the memberships of data points preserved in the symmetric low-rank representation (i.e., the angular information of the principal directions of the symmetric low-rank representation). Finally, we discuss the convergence properties and present a computational complexity analysis of LRRSC.

\subsection{Low-rank representation with symmetric constraint}
\label{sec:LRRSC1}

When there are no errors in data $X$ (i.e., the data are strictly drawn from $k$ independent subspaces, and already ordered according to their respective clusters) the row space of the data, denoted by ${V^*}{V^{*T}}$ \cite{Wei2011RSI} (also known as the shape interaction matrix \cite{Costeira1998MF}), is a block diagonal matrix that has exactly $k$ blocks. The row space of $X$ can be used as affinity for subspace clustering, and produces good results. It can be calculated using a closed form by computing the skinny SVD of the data matrix $X$, i.e., $X = {U^*}{\Sigma ^*}{V^{*T}}$. However, real observations are often noisy. Grossly corrupted observations may reduce the performance. Low-rank representation techniques can be used to alleviate these problems \cite{Liu2010LRR, Ni2010LRRPSD}.

To guarantee weight consistency for each pair of data points, we impose a symmetric constraint on the low-rank representation. Incorporating low-rank representation with a symmetric constraint in Problem \eqref{eq:lrrwithnoisy}, we consider the following convex optimization problem to seek a symmetric low-rank representation $Z$:
\begin{equation}\label{eq:LRRSC}
\mathop {\min }\limits_{Z,E} {\left\| Z \right\|_*} + \lambda {\left\| E \right\|_{2,1}} \quad s.t. \quad  X = AZ + E, Z = {Z^T},
\end{equation}
where the parameter $\lambda > 0$ is used as a trade-off between low-rankness and the effect of noise. The low-rankness criterion effectively captures the global structure of $X$. The low-rank constraint guarantees that the coefficients of samples coming from the same subspace are highly correlated. We assume that corruptions are "sample-specific", i.e., some data vectors are corrupted and others are not. The ${l_{2,1}}$-norm is used to characterize the error term $E$, because it encourages the columns of $E$ to be zero. For small Gaussian noise, $\left\| E \right\|_F^2$ is an appropriate choice. Each coefficient pair $({z_{ij}},{z_{ji}})$ denotes the interaction between data points ${x_i}$ and ${x_j}$. The symmetric constraint criterion is incorporated into the low-rankness property of high-dimensional data representation so that it can effectively ensure the weight consistency for each pair of data points. Consequently, the symmetric low-rank representation, which preserves the subspace structures of high-dimensional data, ensures that highly correlated data points of subspaces are represented together.

To make the objective function in Problem \eqref{eq:LRRSC} separable, we first convert it to the following equivalent problem by introducing an auxiliary variable $J$:
\begin{equation}\label{eq:LRRSC_ALM}
\mathop {\min }\limits_{Z,E,J} {\left\| J \right\|_*} + \lambda {\left\| E \right\|_{2,1}} \quad s.t. \quad X = AZ + E,Z = J,J = {J^T}.
\end{equation}
The augmented Lagrangian function of Problem \eqref{eq:LRRSC_ALM} is
\begin{equation}
\begin{split}
& \mathop {\min }\limits_{Z,E,J = {J^T},{Y_1},{Y_2}} {\left\| J \right\|_*} + \lambda {\left\| E \right\|_{2,1}} + tr[Y_1^T\left(X - AZ - E\right)] + \\
&tr[Y_2^T\left(Z - J\right)] + \frac{\mu }{2}\left(\left\| {X - AZ - E} \right\|_F^2 + \left\| {Z - J} \right\|_F^2\right),
\end{split}
\end{equation}
where ${Y_1}$ and ${Y_2}$ are Lagrange multipliers, and $\mu > 0$ is a penalty parameter. The above optimization problem can be formulated as follows:
\begin{equation}
\begin{split}
& \mathop {\min }\limits_{Z,E,J = {J^T},{Y_1},{Y_2}} {\left\| J \right\|_*} + \lambda {\left\| E \right\|_{2,1}} + \\
& \frac{\mu }{2}\left( {\left\| {X - AZ - E + \frac{{{Y_1}}}{\mu }} \right\|_F^2 + \left\| {Z - J + \frac{{{Y_2}}}{\mu }} \right\|_F^2} \right).
\end{split}
\end{equation}
This can be effectively solved by inexact ALM \cite{Lin2011ALM}. The variables $J$, $Z$ and $E$ can be updated alternately at each step, while the other two variables are fixed. The updating schemes at each iteration are:
\begin{equation}
\begin{split}
\label{eq:LRRSCInDetail}
& {J_{k + 1}} = \arg \mathop {\min }\limits_{J = {J^T}} \frac{1}{\mu }{\left\| J \right\|_*} + \frac{1}{2}\left\| {J - \left(Z + \frac{{{Y_2}}}{\mu }\right)} \right\|_F^2, \\
& {Z_{k + 1}} = {\left(I + {A^T}A\right)^{ - 1}}\left({A^T}X - {A^T}E + J + \frac{{{A^T}{Y_1} - {Y_2}}}{\mu }\right), \\
& {E_{k + 1}} = \arg \min \lambda {\left\| E \right\|_{2,1}} + \frac{\mu }{2}\left\| {(E - \left(X - AZ + \frac{{{Y_1}}}{\mu }\right)} \right\|_F^2.
\end{split}
\end{equation}

The complete procedure for solving Problem \eqref{eq:LRRSCInDetail} is outlined in Algorithm \ref{alg:LRRSCAlg1}. By choosing a proper dictionary $A$, LRRSC seeks the lowest-rank representation matrix. Because each data point can be represented by the original data points, LRRSC uses the X as its dictionary. The last equation in Problem \eqref{eq:LRRSCInDetail} is a convex problem and has a closed form solution. It is solved using the ${l_{2,1}}$-norm minimization operator \cite{Liu2010LRR1}. The first equation in Problem \eqref{eq:LRRSCInDetail} is solved using the following lemma:

\begin{algorithm}
\renewcommand{\algorithmicrequire}{\textbf{Input:}}
\renewcommand\algorithmicensure {\textbf{Output:} }
\caption{Solving Problem \eqref{eq:LRRSC} by Inexact ALM}
\label{alg:LRRSCAlg1}
\begin{algorithmic}
\REQUIRE ~~\\
data matrix $X$, parameters $\lambda > 0$
\end{algorithmic}
{\bfseries Initialize:}
$Z=J=0, E=0, {Y_1}={Y_2}=0, \mu={10^{-2}}, {\mu_{\max }}={10^{10}},\rho=1.1,\varepsilon = {10^{ - 6}}$
\begin{algorithmic}[1]
\WHILE {not converged}
\STATE update the variables as \eqref{eq:LRRSCInDetail}; \\
\STATE update the multipliers: \\
\begin{center}
${Y_1} = {Y_1} + {X} - {X}Z - E$; \\
${Y_2} = {Y_2} + {X} - \mu \left(Z - J\right)$;
\end{center}
\STATE update the parameter $\mu$ by $\mu= \min (\rho \mu,{\mu _{\max}})$;
\STATE check the convergence conditions \\
\begin{center}
${\left\| {X - AZ - E} \right\|_\infty} < \varepsilon $ and ${\left\| {Z - J} \right\|_\infty} < \varepsilon $;
\end{center}
\ENDWHILE
\ENSURE ~~\\ ${Z^*},{E^*}$
\end{algorithmic}
\end{algorithm}

\begin{lemma}
\label{lemma1}
Given any square matrix $Q \in {\mathbf{R}^{n \times n}}$, the unique closed form solution to the optimization problem

\begin{equation}
{W^*} = \arg \mathop {\min }\limits_W \frac{1}{\mu }{\left\| W \right\|_*} + \frac{1}{2}\left\| {W - Q} \right\|_F^2,W = {W^T},
\end{equation}
takes the form
\begin{equation}
{W^*} = {U_r}\left({\Sigma _r} - \frac{1}{\mu } \cdot {{\rm I}_r}\right)V_r^T,
\end{equation}
 where $\widetilde Q = U\Sigma {V^T}$ is the skinny SVD of the symmetric matrix $\widetilde Q = {{(Q + {Q^T})} \mathord{\left/
 {\vphantom {{(Q + {Q^T})} 2}} \right.
 \kern-\nulldelimiterspace} 2}$, ${\Sigma _r} = diag\left({\sigma _1},{\sigma _2},...,{\sigma _r}\right)$ with $\{ r:{\sigma _r} > \frac{1}{\mu }\} $ are positive singular values, ${U_r}$ and ${V_r}$ are the corresponding singular vectors of the matrix $\widetilde Q$, and ${{\rm I}_r}$ is an $r \times r$ identity matrix.
\end{lemma}

The proof of Lemma \ref{lemma1} is based on the following three lemmas.

\begin{lemma}
\label{lemma2}
(\cite{Liu2010LRR} Lemma 7.1)
Let $U$, $V$ and $W$ be matrices with compatible dimensions. Suppose both $U$ and $V$ have orthogonal columns, i.e., ${U^T}U = I$ and ${V^T}V = I$. Then we have
\begin{equation}
{\left\| W \right\|_*} = \left\| {UW{V^T}} \right\|{}_*.
\end{equation}
Note that a similar equality also holds for the square of the Frobenius norm, $\left\| \cdot \right\|_F^2$.
\end{lemma}

\begin{lemma}
\label{lemma3}
(\cite{Liu2010LRR1} Lemma 3.1)
Let $A$ and $D$ be square matrices, and let $B$ and $C$ be matrices with compatible dimensions. Then, for any block partitioned matrix $X = \left[ {\begin{array}{*{20}{c}}
A&B\\
C&D
\end{array}} \right]$,
\begin{equation}
{\left\| {\left( {\begin{array}{*{20}{c}}
A&B\\
C&D
\end{array}} \right)} \right\|_*} \ge {\left\| {\left( {\begin{array}{*{20}{c}}
A&0\\
0&D
\end{array}} \right)} \right\|_*} = {\left\| A \right\|_*} + {\left\| D \right\|_*}.
\end{equation}
Note that a similar inequality also holds for the square of the Frobenius norm, $\left\| \cdot \right\|_F^2$ \cite{Ni2010LRRPSD}.
\end{lemma}

\begin{lemma}
\label{lemma4} (\cite{Tibshiran1996Lasso, Chen2001BP, Hale2008FixedL1, Parikh2013PA})
Let $A \in {\mathbf{R}^{m \times n}}$ be a given matrix and ${\left\| \cdot \right\|_1}$ be the ${l_1}$-norm. The proximal operator of
\begin{equation}
\mathop {\min }\limits_x \frac{1}{2}\left\| {Ax - b} \right\|_2^2 + \gamma {\left\| x \right\|_1}
\end{equation} is ${S_\gamma }(x)$, i.e. the soft-thresholding operator
\[{S_\gamma }(x) = \left\{{\begin{array}{*{20}{c}}
{x - \gamma , \quad x > \gamma }\\
{x + \gamma , \quad x <  - \gamma }\\
{0,\quad else}
\end{array}} \right. .\]
\end{lemma}

\begin{proof}(of Lemma \ref{lemma1} )
Let ${W^*}$ be the unique minimizer. Then,
\begin{equation}
\begin{split}
& \left\| {{W^*} - Q} \right\|_F^2 = \frac{1}{2}\left( {\left\| {{W^*} - Q} \right\|_F^2 + \left\| {{W^*} - {Q^T}} \right\|_F^2} \right) \\
& = \frac{1}{2}\left( {\left\| {{W^*} - {{(Q + {Q^T})} \mathord{\left/ {\vphantom {{(Q + {Q^T})} 2}} \right. \kern-\nulldelimiterspace} 2}} \right\|_F^2} \right) + {\rm F}(Q),
\end{split}
\end{equation}
where ${\rm F}(Q)$ are entirely independent of ${W^*}$. Therefore, the original optimization can be converted to
\begin{equation}
\label{eq:LRRV1}
{W^*} = \arg \mathop { \min }\limits_W \frac{1}{\mu }{\left\| W \right\|_*} + \frac{1}{2}\left\| {W - \widetilde Q} \right\|_F^2, W = {W^T},
\end{equation}
where $\widetilde Q = {{\left(Q + {Q^T}\right)} \mathord{\left/
 {\vphantom {{(Q + {Q^T})} 2}} \right.
 \kern-\nulldelimiterspace} 2}$.

 Let $\widetilde Q = {U_r}{\Sigma _r}V_r^T$ be the skinny SVD of a given symmetric matrix $\widetilde Q$ of rank $r$, where ${\Sigma _r} = diag({\sigma _1},{\sigma _2},...,{\sigma _r})$ with $\{ r:{\sigma _r} > \frac{1}{\mu }\} $ are positive singular values. More precisely, both ${U_r}$ and ${V_r}$ have orthogonal columns, i.e., $U_r^T{U_r} = I$ and $V_r^T{V_r} = I$. Let $W = {U_r}\widetilde WV_r^T$. By Lemma \ref{lemma2}, Problem \eqref{eq:LRRV1} is equivalent to
 \begin{equation}
\label{eq:LRRV2}
{\widetilde W^*} = \arg \mathop { \min }\limits_{\widetilde W} \frac{1}{\mu }{\left\| {\widetilde W} \right\|_*} + \frac{1}{2}\left\| {\widetilde W - {\Sigma _r} } \right\|_F^2,\widetilde W = {\widetilde W^T}.
\end{equation}

Let ${\widetilde W^*}$ be an optimizer of Problem \eqref{eq:LRRV2}, then ${\widetilde W^*}$ must be a diagonal matrix. Assume ${\widetilde W_0}$ is a non-diagonal matrix, i.e., there exists nonzero entries in the non-diagonal of ${\widetilde W^*}$. Let $f\left(\widetilde W\right) = \frac{1}{\mu }{\left\| {\widetilde W} \right\|_*} + \frac{1}{2}\left\| {\widetilde W - {\Sigma _r} } \right\|_F^2$. By Lemma \ref{lemma3} and the strict decreasing property of the square of the Frobenius norm, we can always derive ${\widetilde W_1}$ by removing the non-diagonal entries of ${\widetilde W_0}$ such that $f\left({\widetilde W_1}\right) < f\left({\widetilde W_0}\right)$. This contradicts the non-diagonal matrix assumption.

Let $\widetilde W = diag({w_1},{w_2},...,{w_r})$. For diagonal matrices, the sum of the singular values is equal to the sum of the absolute values of the diagonal elements, i.e., the ${l_1}$-norm. The optimization of Problem \eqref{eq:LRRV2} is equivalent to
 \begin{equation}
{\widetilde W^*} = \arg \mathop { \min }\limits_{\widetilde W} \frac{1}{\mu }{\left\| {\widetilde W} \right\|_1} + \frac{1}{2}\left\| {\widetilde W - {\Sigma _r} } \right\|_F^2,\widetilde W = {\widetilde W^T}.
\end{equation}

By Lemma \ref{lemma4}, the optimal solution to this problem is given by ${\widetilde W^*} = {S_{\frac{1}{\mu }}}\left({\Sigma _r}\right)$. Finally, we get ${W^*} = {U_r}\left({\Sigma _r} - \frac{1}{\mu } \cdot {{\rm I}_r}\right)V_r^T$.

\end{proof}

\subsection{Building an affinity graph based on the symmetric low-rank matrix}
\label{sec:SM}

The critical task of subspace clustering in this paper mainly focus on how to learn a good affinity matrix in a time efficient manner. Using the optimal symmetric low-rank matrix ${Z^*}$ from Problem \eqref{eq:LRRSC}, we need to construct an affinity graph $G = (V,E)$ associated with the corresponding adjacency matrix ${\rm{W  =  \{ }}{{\rm{w}}_{ij}}|i,j \in V{\rm{\} }}$, where $V = \{ {v_1},{v_2}...,{v_n}\}$ is a set of vertices and $E = \{ {e_{ij}}|i,j \in V\}$ is a set of edges. Note that $\{ {w_{ij}}\}$ represents the weight of edge ${e_{ij}}$ that associates vertices $i$ and $j$. A fundamental problem involving affinity graph construction is how to determine the adjacency matrix $W$. Because each sample is represented by the others, each element ${z_{ij}}$ of the matrix ${Z^*}$ naturally characterizes the contribution of the sample ${x_j}$ to the reconstruction of sample ${x_i}$.

A straightforward method is to use the symmetric low-rank matrix ${Z^*}$ as the adjacency matrix $W$ for spectral clustering. It can mostly achieve satisfied results. However, a few entries of the matrix are sensitive to the noise or outliers in low-rank representation. Those entries cannot reflect the real relationships of the pairwise points. The matrix ${Z^*}$ cannot adequately represent the relationship between samples when there are grossly corrupted observations. Therefore, the straightforward use of the matrix ${Z^*}$ as a pairwise affinity relationship between data points inevitably results in loss of information, and it does not fully capture the intrinsic correlation of data points \cite{Agarwal2005Clustering, Zhou2006HyperGraph}. Consequently, the clustering performance may seriously decline because of a lack of robustness.

Several existing methods that use the angular information of original samples to measure the relationship between samples were proposed in literatures \cite{Lauer2009SC, Luxburg2007SC}. However, a lack of robustness is inevitable because of corrupted samples. To enhance the ability of the low-rank representation to separate samples in different subspaces, a reasonable strategy is to derive an affinity graph with enhanced clustering information from the matrix ${Z^*}$. This preserves the subspace structures of high-dimensional data, and recovers the clustering relations among samples. If any two data points $x_i$ and $x_j$ are close in the intrinsic geometry of the data distribution, then the representations of these two points, namely, $z_i$ and $z_j$ with respect to the same basis $X$, are close to each other. It has been demonstrated in recent studies of manifold learning theory \cite{Deng2011GRNMF}. To remove the effect of various noises in the high-dimensional data, we employed a mechanism of exploiting the angular information of its principal directions instead of the low-rank representation itself. As the coefficient matrix is low-rank, the angular information can hardly be effected by the few error entries of rows or column in the low-rank coefficient matrix. In other words, the elements of the affinity matrix are more robust to the noise or outliers than those of the low-rank coefficient matrix. Hence, the angular information of its principal directions is a good surrogate for symmetric low-rank representation.

Instead of directly using $\left| {{Z^*}} \right| + \left| {{{\left({Z^*}\right)}^T}} \right|$ to define an affinity graph, we consider the mechanism driving the construction of the affinity graph from the matrix ${Z^*}$. We consider ${Z^*}$ with the skinny SVD ${U^*}{\sum ^*}{\left({V^*}\right)^T}$. Note that ${U^*}$ and ${V^*}$ are the orthogonal bases of the column and row space of the matrix ${Z^*}$. Inspired by \cite{Lauer2009SC, Liu2010LRR}, we assign each column of ${U^*}$ a weight by multiplying it by ${({\sum ^*})^{{1 \mathord{\left/ {\vphantom {1 2}} \right. \kern-\nulldelimiterspace} 2}}}$, and each row of ${\left({V^*}\right)^T}$ a weight by multiplying it by ${({\sum ^*})^{{1 \mathord{\left/ {\vphantom {1 2}} \right. \kern-\nulldelimiterspace} 2}}}$. Then, we can define $M = {U^*}{\left({\sum ^*}\right)^{{1 \mathord{\left/
 {\vphantom {1 2}} \right.
 \kern-\nulldelimiterspace} 2}}},N = {({\sum ^*})^{{1 \mathord{\left/
 {\vphantom {1 2}} \right.
 \kern-\nulldelimiterspace} 2}}}{\left({V^*}\right)^T}$. The product of two matrices can represent ${Z^*}$, i.e., ${Z^*} = MN$.

Because ${Z^*}$ is a symmetric low-rank matrix, the absolute values of the columns of ${U^*}$ and ${V^*}$ always agree. The columns of ${U^*}$ or ${V^*}$ can span the principal directions of the symmetric low-rank matrix ${Z^*}$, and the diagonal entries reflect the relative importance of the coefficient matrix ${Z^*}$ in each of these directions. Therefore, we use angular information from all of the row vectors of matrix $M$, or all of the column vectors of matrix $N$, instead of ${Z^*}$, to define an affinity matrix $W$ as follows:
\begin{equation}
{[W]_{ij}} = {\left( {\frac{{m_i^T{m_j}}}{{{{\left\| {{m_i}} \right\|}_2}{{\left\| {{m_j}} \right\|}_2}}}} \right)^{2\alpha }}
\quad or \quad
{[W]_{ij}} = {\left( {\frac{{n_i^T{n_j}}}{{{{\left\| {{n_i}} \right\|}_2}{{\left\| {{n_j}} \right\|}_2}}}} \right)^{2\alpha }},
\end{equation}
where ${m_i}$ and ${m_j}$ denote the $i$-th and $j$-th row of the matrix $M$, or ${n_i}$ and ${n_j}$ denote the $i$-th and $j$-th column of the matrix $N$. The values of the affinity matrix $W$ can be distributed on the unit ball by the ${l_2}$-norm of data vectors. Consequently, the affinity matrix $W$ preserves angular information between data vectors but removes length information. By using ${( \cdot )^{2\alpha }}$ we ensure that the values of the affinity matrix $W$ are positive inputs for the subspace clustering, and also increase the separation of points of different groups because of the geometry of the ${l_2}$-norm ball. Finally, we apply spectral clustering algorithms such as NCuts \cite{Shi2000Ncuts} to segment the samples into a given number of clusters. Algorithm \ref{alg:LRRSCAlg2} summarizes the complete subspace clustering algorithm of LRRSC.

\begin{algorithm}[!htbp]
\renewcommand{\algorithmicrequire}{\textbf{Input:}}
\renewcommand\algorithmicensure {\textbf{Output:} }
\caption{The LRRSC algorithm}
\label{alg:LRRSCAlg2}
\begin{algorithmic}[1]
\REQUIRE ~~\\
data matrix $X = [{x_1},{x_2},...,{x_n}] \in {\mathbf{R}^{m \times n}}$, number $k$ of subspaces, regularized parameters $\lambda > 0, \alpha > 0 $

\STATE Solving the following problem by Algorithm \ref{alg:LRRSCAlg1}:

\begin{equation*}
\mathop {\min }\limits_{Z,E} {\left\| Z \right\|_*} + \lambda {\left\| E \right\|_{2,1}} \quad s.t. \quad  X = AZ + E, Z = {Z^T},
\end{equation*}

and obtain the optimal solution $({Z^*},{E^*})$.

\STATE Compute the skinny SVD ${Z^*}={U^*}{\sum ^*}{({V^*})^T}$.
\STATE Calculate $M = {U^*}{({\sum ^*})^{{1 \mathord{\left/ {\vphantom {1 2}} \right. \kern-\nulldelimiterspace} 2}}}$ or $N = {({\sum ^*})^{{1 \mathord{\left/ {\vphantom {1 2}} \right. \kern-\nulldelimiterspace} 2}}}{({V^*})^T}$.
\STATE Construct the affinity graph matrix $W$, i.e.,
\begin{equation*}
{[W]_{ij}} = {\left( {\frac{{m_i^T{m_j}}}{{{{\left\| {{m_i}} \right\|}_2}{{\left\| {{m_j}} \right\|}_2}}}} \right)^{2\alpha }}
\quad or \quad
{[W]_{ij}} = {\left( {\frac{{n_i^T{n_j}}}{{{{\left\| {{n_i}} \right\|}_2}{{\left\| {{n_j}} \right\|}_2}}}} \right)^{2\alpha }}.
\end{equation*}
\STATE Apply $W$ to perform NCuts.

\ENSURE ~~\\ The clustering results.
\end{algorithmic}
\end{algorithm}

\subsection{Pursuing a symmetric low-rank representation through a closed form solution}
\label{sec:comparisons}

The existing methods, e.g., a sparsity constraint and rank minimization, for obtaining reasonable data representation to characterize the correlation structure of high-dimensional data have high computational cost. Specifically, LRRSC typically requires iterative SVD operations when solving the nuclear norm optimization problem. Therefore, it suffers from a high computation cost.

In this section, we present an efficient variant of LRRSC, namely eLRRSC, to improve the scalability of LRRSC, which attempts to learn the symmetric low-rank representation by a closed form solution without iterative SVD operations. In particular, the alternative learning scheme is composed of three steps.

First, eLRRSC uses a collaborative representation with regularized least square for symmetric representation, which is also adopted in SLRR \citet{Chen2016SLRR}. The optimization problem is formalized as follows:
\begin{equation}\label{eq:LRRRLS1}
\mathop {\min }\limits_Z \left\| Z \right\|_F^2 + \frac{\lambda }{2}\left\| {X - XZ} \right\|_F^2 \quad s.t. \quad X = XZ+E,
\end{equation}
The above problem has a closed form solution:
\begin{equation}
Z = {\left({X^T}X + \lambda \cdot I\right)^{ - 1}}{X^T}X,
\end{equation}
where $\lambda > 0$ is a parameter and $I$ is the identity matrix of size $n \times n$.

By utilizing the self-expressiveness property of the data,  we consider a general model of data representation:
\begin{equation} \label{eq:generalmodel}
\min f\left(Z\right) \quad s.t. \quad X = XZ + E,
\end{equation}
where $f(Z)$ is a matrix function, i.e., $\left\|  \cdot  \right\|_F^2$ or $\left\| Z  \right\|{_*}$, and $E$ is an error term. The optimal solution ${Z^*}$ is a particular representation of data $X$. Then we can write
\begin{equation} \label{eq:generalequation}
A = XZ,
\end{equation}
where each column of the data matrix $A = \{ {a_1},{a_2},...,{a_n}\} \in {\mathbf{R}^{m \times n}}$ represents a corrected sample, i.e., ${a_i} = X{z_i}$. In other words, each corrected sample in a union of subspaces can be efficiently reconstructed by other original samples in the dataset. The underlying assumption for the success of the eLRRSC algorithm is that the samples are drawn from the union of low-dimensional subspaces. As a result, the matrix $A$ should be low-rank. Denote the ranks of $A$, $X$ and $Z$ by $rank(A)$, $rank(X)$ and $rank(Z)$, respectively. Therefore $Z$ is low-rank as $rank(A) \le min(rank(X),rank(Z))$. Consequently, we need to construct the symmetric low-rank representation instead of the collaborative representation to characterize the correlation structure of high-dimensional data. As mentioned above, the symmetric low-rank representation of high-dimensional data play the essential role for preserving the subspace structures. Hence, we obtain an alternative symmetric low-rank ${Z^{'}}$ instead of $Z$ from a collaborative representation of high-dimensional data by solving the following optimization problem:
\begin{equation}\label{eq:LRRRLS2}
\mathop {\min }\limits_W {\left\| {{Z^{'}}} \right\|_*} + \frac{\mu }{2}\left\| {{Z^{'}} - Z} \right\|_F^2\quad s.t.\quad {Z^{'}} - Z = E,{Z^{'}} = {Z^{'}}^T.
\end{equation}
The above problem can be solved by Lemma \ref{lemma1}. Given the assumption that high-dimensional data are approximately drawn from a union of multiple subspaces, it is reasonable that eLRRSC considers a symmetric low-rank representation of high-dimensional data as a good surrogate for the collaborative representation. It is clear that eLRRSC obtained a symmetric low-rank representation through a closed form solution. However, eLRRSC does not pursue a symmetric low-rank matrix for the low-rank matrix recovery or completion. Instead, it mainly focuses on only representations of data, which is further applied to evaluate the membership between samples.

Finally, we make use of the angular information of its principal directions to get an affinity matrix after obtaining the symmetric low-rank ${Z^{\ast}} $. The angular information can be applied in the spectral clustering algorithm, such as NCuts \cite{Shi2000Ncuts}, to produce the final clustering results. Algorithm \ref{alg:eLRRSC} summarizes the complete subspace clustering algorithm of eLRRSC. The purpose of the first two steps is to pursue a symmetric low-rank representation of original data with respect to the same basis $X$ in Algorithm \ref{alg:eLRRSC}. By making use of a closed-form solution, eLRRSC effectively provides an alternative scheme instead of LRRSC to seek a low-rank matrix, which preserves the intrinsically geometrical structure of the memberships of samples.

\begin{algorithm}[!htbp]
\renewcommand{\algorithmicrequire}{\textbf{Input:}}
\renewcommand\algorithmicensure {\textbf{Output:} }
\caption{The eLRRSC algorithm}
\label{alg:eLRRSC}
\begin{algorithmic}[1]
\REQUIRE ~~\\
data matrix $X = [{x_1},{x_2},...,{x_n}] \in {\mathbf{R}^{m \times n}}$, number $k$ of subspaces, regularized parameters $\lambda > 0$, $\mu > 0$ and $\alpha > 0 $

\STATE Solving the following problem:
\begin{equation*}
\mathop {\min }\limits_Z \left\| Z \right\|_F^2 + \frac{\lambda }{2}\left\| {X - XZ} \right\|_F^2 \quad s.t. \quad X = XZ+E,
\end{equation*}
and obtain the optimal solution ${Z} = {\left({X^T}X + \lambda \cdot I\right)^{ - 1}}{X^T}X$.

\STATE Solving the following problem by Lemma \ref{lemma1}:
\begin{equation*}
\mathop {\min }\limits_W {\left\| {{Z^{'}}} \right\|_*} + \frac{\mu }{2}\left\| {{Z^{'}} - Z} \right\|_F^2\quad s.t.\quad {Z^{'}} - Z = E,{Z^{'}} = {Z^{'}}^T
\end{equation*}
and obtain the optimal solution ${Z^{'} = U_r}SV_r^T$, where $Z = U\Sigma {V^T}$, $S={\Sigma _r} - \frac{1}{\mu } \cdot {I_r}$, ${\Sigma _r} = diag({\sigma _1},{\sigma _2},...,{\sigma _r})$ with $\{ r:{\sigma _r} > \frac{1}{\mu }\} $ are positive singular values.

\STATE Calculate $M = {U_r}{S^{{1 \mathord{\left/ {\vphantom {1 2}} \right. \kern-\nulldelimiterspace} 2}}}$ or $
N = {S^{{1 \mathord{\left/
 {\vphantom {1 2}} \right.
 \kern-\nulldelimiterspace} 2}}}V_r^T$

\STATE Construct the affinity graph matrix $W$, i.e.,
\begin{equation*}
{[W]_{ij}} = {\left( {\frac{{m_i^T{m_j}}}{{{{\left\| {{m_i}} \right\|}_2}{{\left\| {{m_j}} \right\|}_2}}}} \right)^{2\alpha }}
\quad or \quad
{[W]_{ij}} = {\left( {\frac{{n_i^T{n_j}}}{{{{\left\| {{n_i}} \right\|}_2}{{\left\| {{n_j}} \right\|}_2}}}} \right)^{2\alpha }}.
\end{equation*}
\STATE Apply $W$ to perform NCuts.

\ENSURE ~~\\ The clustering results.
\end{algorithmic}
\end{algorithm}

\subsection{Relationship between eLRRSC and SLRR}
\label{sec:slrr}

We emphasize that eLRRSC is the follow-up research based on our previous work, i.e., SLRR \cite{Chen2016SLRR}. First, eLRRSC and SLRR share the same collaborative representation with regularized least square. Then, both of them utilize the angular information of principal directions of the symmetric low-rank representation to construct the affinity matrix for the spectral clustering algorithm. However, there is a major difference in constructing a desirable low-rank symmetric matrix. In fact, each of them employ entirely different idea to obtain their own symmetric low-rank matrices.

To obtain the symmetric low-rank matrix, SLRR attempts to pursue an alternative low-rank matrix instead of the original data matrix through existing low-rank matrix recovery techniques. To achieve competitive subspace clustering performance, SLRR need to elaborate a proper alternative low-rank matrix, which highly depends on the choice of low-rank matrix recovery techniques. This means that SLRR requires more prior knowledge of original data, i.e., types of noise. Besides, the complexity of the low-rank matrix recovery technique adopted by SLRR also may lead to high computational cost.

On the other hand, eLRRSC obtains the symmetric low-rank matrix from the collaborative representation under the assumption that high-dimensional data involves with the multiple subspace structures, without considering the noise types of original data. Hence, the overall computational cost of eLRRSC can be effectively guaranteed. Besides, the key observation is that the intrinsically geometrical structure of the samples' memberships are preserved in the symmetric low-rank representation, which is closely related to the new basis, i.e., the original data. This shows that both of them essentially focus on different construction methods on symmetric low-rank matrices respectively although a closed form solution improves the computational efficiency of large-scale subspace clustering.

\subsection{Convergence properties and computational complexity analysis}
\label{sec:cpcca}

The convergence properties of the exact ALM algorithm for a smooth objective function have been generally proven in \cite{Lin2011ALM}. The inexact variation of ALM has been extensively studied and generally converges well. Algorithm \ref{alg:LRRSCAlg2} performs well in practical applications. We assume that the size of $X$ is $m \times n$, where $X$ has $n$ samples and each sample has $m$ dimensions. The computational complexity of the first step in Algorithm \ref{alg:LRRSCAlg1} is $O({n^3})$ because it requires computing the SVD of a $n \times n$ matrix. The overall computational complexity of Algorithm \ref{alg:LRRSCAlg1} is $O(2{n^3} + m{n^2})$. When $n > m$, the computational complexity of Algorithm  \ref{alg:LRRSCAlg1} can be considered to be $O({n^3})$. The computational complexity of Algorithm \ref{alg:LRRSCAlg2} is $O(t{n^3}) + O({n^3})$, where $t$ is the number of iterations. Therefore, the final overall complexity of Algorithm \ref{alg:LRRSCAlg2} is $O(t{n^3})$. In our experiments, there were always less than 335 iterations.

In Algorithm \ref{alg:eLRRSC}, the computational complexity of the first two steps and the last two steps are $O({n^3})$ and $O({n^2})$, respectively. Therefore, the general complexity is $O({n^3})$. With this theoretical result, we can say that eLRRSC is significantly computational efficient than LRRSC.

\section{Experiments}
\label{sec:Experiments}

In this section, we evaluated the performance of the proposed LRRSC algorithm and its variant (eLRRSC) using a series of experiments on publicly available databases, the extended Yale B, Hopkins 155 and penguin databases (the Matlab source code of our method is available at http://www.machineilab.org/users/chenjie). We compared the performance of LRRSC with several state-of-the-art subspace clustering algorithms (LRR \cite{Liu2010LRR}, LRR-PSD \cite{Ni2010LRRPSD}, SSC \cite{Elhamifar2013SSC}, low rank subspace clustering (LRSC) \cite{Favarob2011LRSC, Vidala2013LRSC} and SLRR \cite{Chen2016SLRR}. As there is no source code publicly available for LRR-PSD, we implemented the LRR-PSD algorithm according to its theory. For LRSC, we chose the noisy data version of (${P_3}$) as its instance. For the other algorithms, we used the source codes provided by their authors.

\begin{table*}[!htbp]
\small
\setlength{\abovecaptionskip}{0pt}
\setlength{\belowcaptionskip}{10pt}
\setlength{\tabcolsep}{3pt}
\centering
\caption{Parameter settings of different algorithms on face clustering. The parameter $\lambda$ is used as a trade-off between low rankness (or sparsity) and the effect of noise (LRRSC, SLRR, LRR, LRR-PSD, SSC). The parameter $\alpha$ enhances the separate ability of low-rank representation between samples in different subspaces used by LRRSC. For SLRR, $n$ is the number of subspaces, i.e., the number of subjects. For LRSC, $\tau$ and  $\lambda$ are two parameters weighting noise.}
\label{parmater1}
\begin{tabular}{|c|c|c|}
\hline
Method & Scenario 1 &  Scenario 2 \\
\hline
 LRRSC & $ \lambda = 0.2, \alpha = 4 $& $ \lambda = 0.1,\alpha = 3 $ \\
\hline
 eLRRSC & $ \lambda = 35, \mu = 1{e^{ - 3}}, \alpha = 3 $ & $ \lambda = 40, \mu = 0.1, \alpha = 2 $ \\
\hline
SLRR & $\alpha = 3, \lambda = 30 $ & $\alpha = 3, \lambda = 1, r = 10n $ \\
\hline
  LRR & \multicolumn{2}{c|}{$ \lambda = 0.18, \alpha = 2 $ } \\
\hline
 LRR-PSD & $ \lambda = 0.2, \alpha = 4 $ & $ \lambda = 0.1,\alpha = 3 $ \\
\hline
LRSC & $\tau  = 0.4, \lambda  = 0.045, \alpha = 2 $ & $\tau  = 0.045, \lambda  = 0.045, \alpha = 3$ \\
\hline
  SSC & ${\lambda _e} = {{8} \mathord{\left/
 {\vphantom {{800} {{\mu _e}}}} \right.
 \kern-\nulldelimiterspace} {{\mu _e}}}$  & ${\lambda _e} = {{20} \mathord{\left/
 {\vphantom {{800} {{\mu _e}}}} \right.
 \kern-\nulldelimiterspace} {{\mu _e}}}$ \\
\hline
\end{tabular}
\end{table*}

\begin{table*}[!htbp]
\small
\setlength{\abovecaptionskip}{0pt}
\setlength{\belowcaptionskip}{10pt}
\setlength{\tabcolsep}{2pt}
\centering
\caption{Parameter settings of different algorithms on motion segmentation.}
\label{parmater2}
\begin{tabular}{|c|c|c|c|}
\hline
\multirow{2}{*}{Method} & \multicolumn{2}{c|}{The Hopkins 155 motion database} & The penguin motion database\\
\cline{2-4}
 &  Scenario 1 & Scenario 2 & Scenario 1\\
\hline
 LRRSC & $\lambda = 3.3, \alpha = 2$ & $\lambda = 3, \alpha = 3$ & $\lambda = 0.05, \alpha = 4$ \\
\hline
 eLRRSC & $ \lambda = 5{e^{ - 3}}, \mu = 0.2, \alpha = 2 $  &  $ \lambda = 5{e^{ - 3}}, \mu = 0.1, \alpha = 2 $ & $\lambda = 5.5, \alpha = 4, \mu = 30 $ \\
\hline
SLRR & $\alpha = 2, \lambda = 5{e^{ - 3}}$ & {$ \alpha = 2, \lambda = 5{e^{ - 3}}, r=4n $} & $ \lambda = 5{e^{ - 3}}, \alpha = 4 $ \\
\hline
  LRR & \multicolumn{2}{c|}{$ \lambda = 4, \alpha  = 2 $} & $ \lambda = 0.1, \alpha  = 2 $ \\
\hline
 LRR-PSD & $\lambda = 3.3, \alpha = 2$ & $\lambda = 3, \alpha = 3$ & $\lambda = 0.05, \alpha = 4$ \\
\hline
LRSC & \multicolumn{2}{c|}{$\tau  = 420, \lambda  = 5000, \alpha  = 2$} & $\tau  = 5, \lambda  = 10, \alpha = 2 $\\
\hline
  SSC & \multicolumn{2}{c|}{${\lambda _z} = {{800} \mathord{\left/
 {\vphantom {{800} {{\mu _z}}}} \right.
 \kern-\nulldelimiterspace} {{\mu _z}}}$} & ${\lambda _z} = {{240} \mathord{\left/
 {\vphantom {{240} {{\mu _z}}}} \right.
 \kern-\nulldelimiterspace} {{\mu _z}}}$\\
\hline
\end{tabular}
\end{table*}

We evaluate the performance of above algorithms by comparing subspace clustering errors:
\begin{equation}
error = \frac{{{N_{error}}}}{{{N_{total}}}},
\end{equation}
where ${{N_{error}}}$ represents the number of misclassified points, and ${{N_{total}}}$ is the total number of points.
LRRSC requires two parameters, $\lambda$ and $\alpha$. Empirically speaking, the parameter $\lambda$ should be relatively large if the data are ``clean'', or smaller if they are contaminated with small noises, and the parameter $\alpha$ ranges from 2 to 4. For the other algorithms, we used the parameters given by the respective authors, or manually tuned the parameters to find the best results. The parameters for these methods are shown in Table \ref{parmater1} and \ref{parmater2}. All the algorithms are implemented by Matlab R2011b, and all experiments are performed on a Windows platform with Intel Core i5-2300 CPU and 16 GB memory.

\subsection{Experiments on face clustering}

Given a collection of face images from multiple individuals, which have various illumination conditions and expression, we'd like to cluster images according to their individuals. Since this set of face images lie close to a union of 9-dimensional subspaces \cite{Basri2003}, the face clustering problem can be boiled down to image clustering problem over a union of subspaces.

In this section, we consider the Extended Yale B Database \cite{Lee05YALEB, GeBeKr01YALEB} for the face clustering problem. This database consists of 2414 frontal images from 38 individuals. There are approximately $59-64$ images available for each person, shooted under various laboratory-controlled lighting conditions. Fig. \ref{fig:yaleba} shows some sample images. In order to improve the efficiency of the experiments, without losing generality, we first resize all images to $48 \times 42$ pixels, therefore each image can be regarded as a vector of 2016 dimensions. In the rest of this section, we will consider two different experimental scenarios to evaluate the performance of our proposed methods.


\begin{figure*}[!htbp]
\centering
\subfigure[The original sample images]{
\label{fig:yaleba}
\includegraphics[width=9cm]{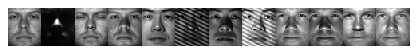}}
\subfigure[The corrupted sample images with the 10\% random pixel corruptions]{
\label{fig:yalebb}
\includegraphics[width=9cm]{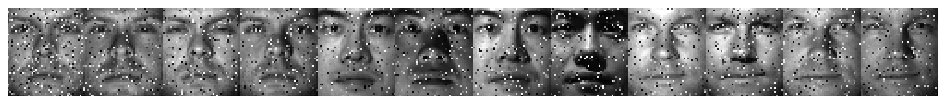}}
\label{fig:yaleb}
\caption{Example images of multiple individuals from the Extended Yale B database.}
\end{figure*}

\texttt{1. First experimental scenario: } We used the first 10 classes of the Extended Yale B Database, as in \cite{Liu2010LRR1}. This subset of the database contains 640 frontal face images from 10 subjects. To compare the clustering errors between different approaches, we first used the raw pixel values without pre-processing, considering each image as a data vector of 2016 dimensions. Then, we applied PCA to pre-process these face images using 100 and 500 feature dimensions.

We consider 640 frontal face images belonging to 10 subjects. Besides using these raw images, we also apply PCA to project these images to 100 and 500 dimension feature spaces, respectively.

\begin{figure*}[!htbp]
\begin{minipage}[t]{0.33\linewidth}
\centering
\subfigure[]{
\label{fig:face:a} 
\includegraphics[width=6cm]{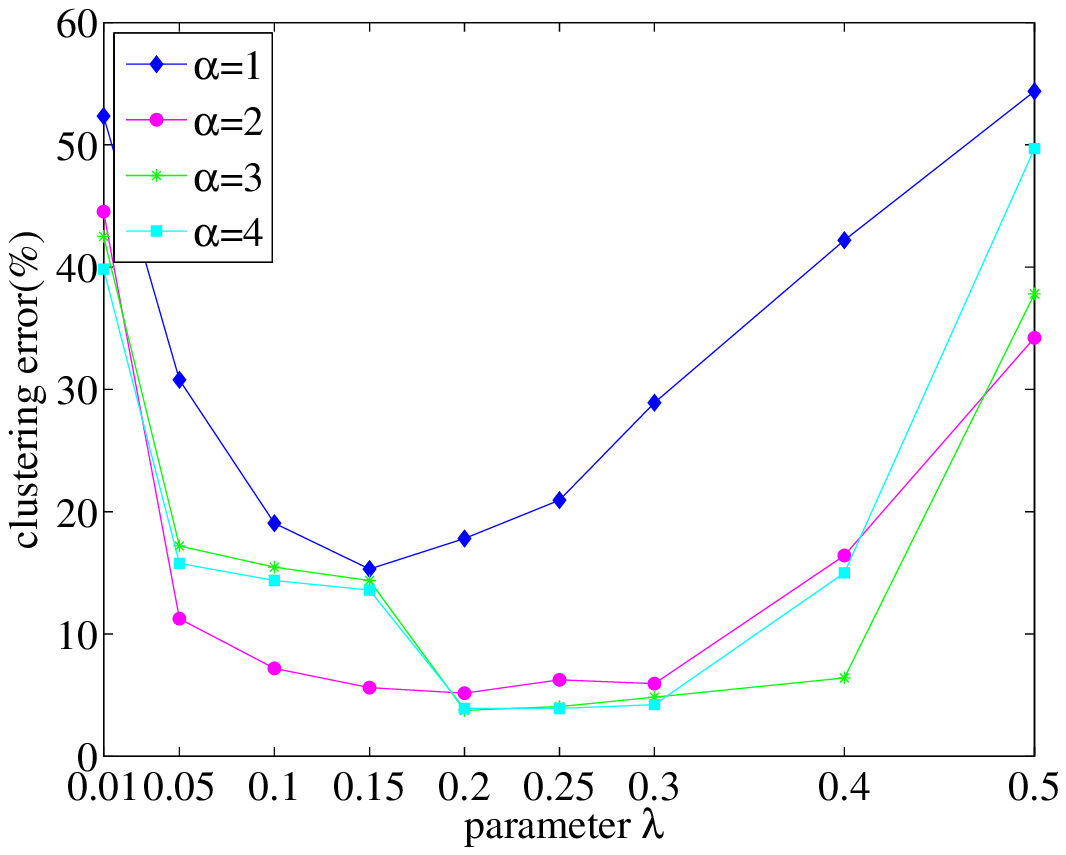}}
\end{minipage}%
\begin{minipage}[t]{0.33\linewidth}
\centering
\subfigure[]{
\label{fig:face:b} 
\includegraphics[width=6cm]{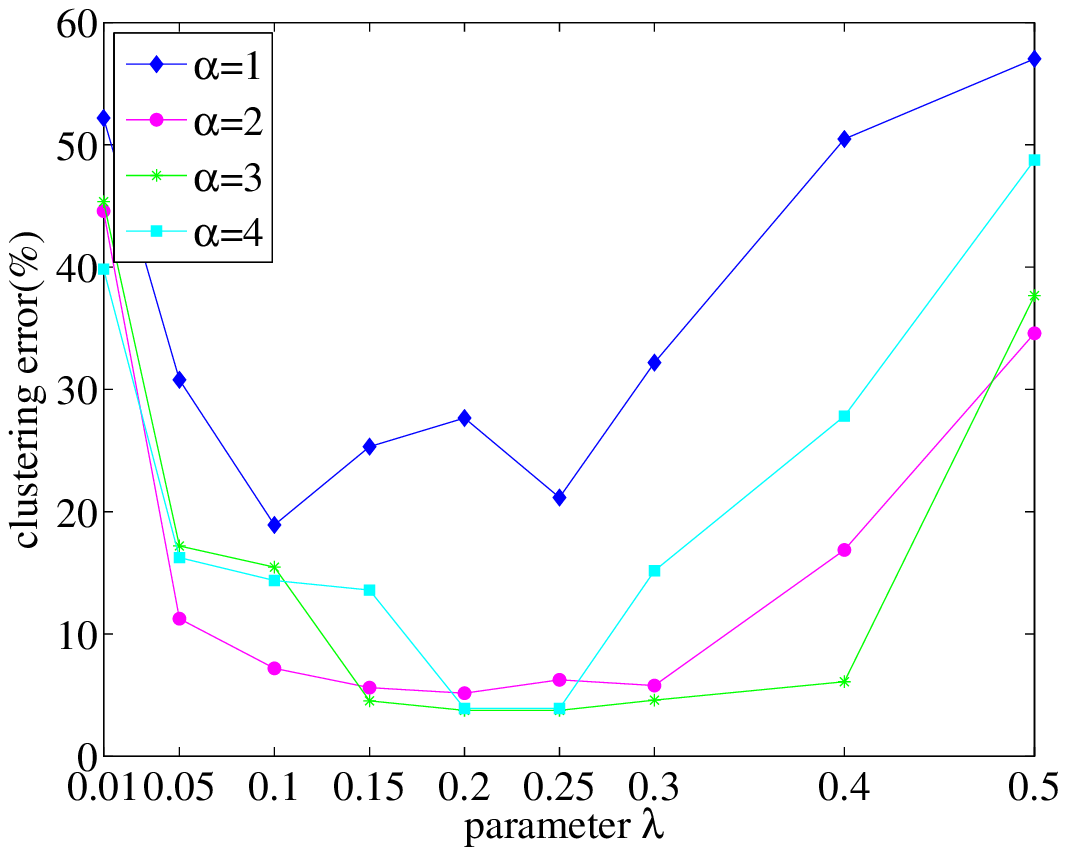}}
\end{minipage}
\begin{minipage}[t]{0.33\linewidth}
\centering
\subfigure[]{
\label{fig:face:c} 
\includegraphics[width=6cm]{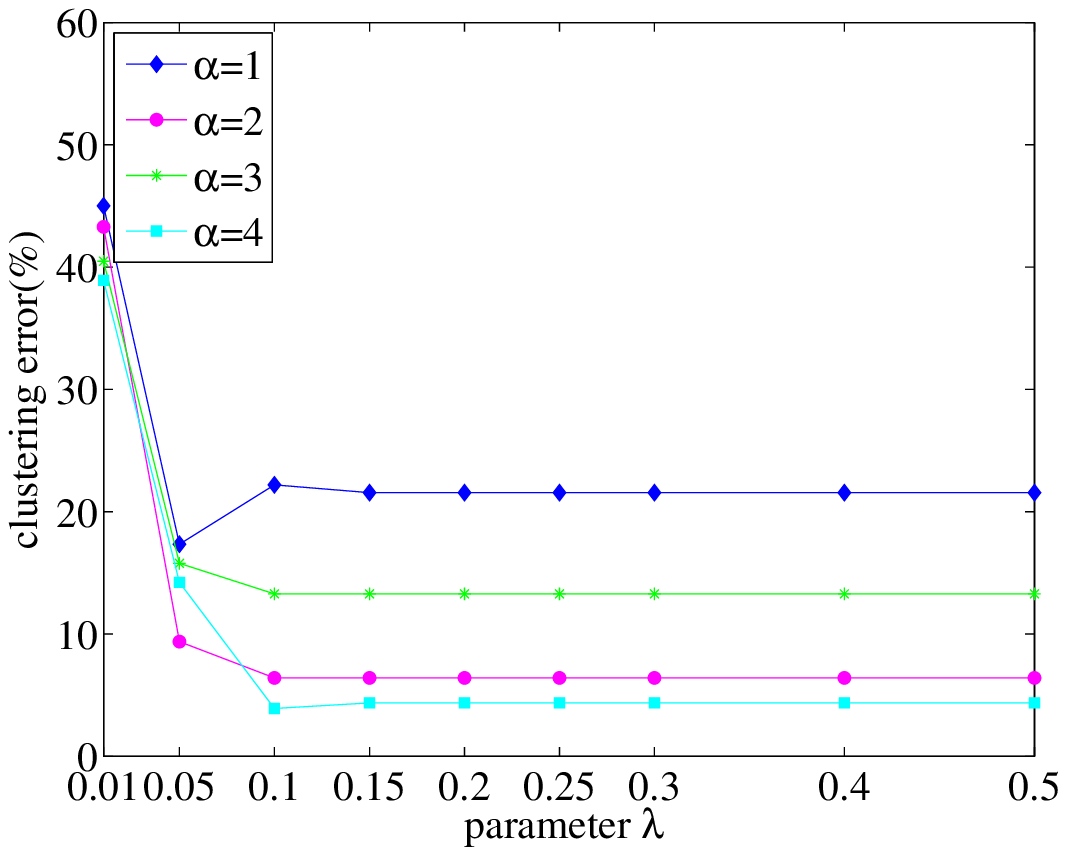}}
\end{minipage}
\caption{Clustering error given different $\lambda$ and $\alpha$ combinations, using the first 10 classes in the Extended Yale B Database. (a) Raw data. (b) The 500-dimensional data obtained by applying PCA. (c) The 100-dimensional data obtained by applying PCA. }
\label{fig:face11} 
\end{figure*}

Fig. \ref{fig:face11} shows the influence of the parameters $\lambda$ and $\alpha$ on the face clustering errors of LRRSC. In each experiment, $\alpha  \in {\rm{\{ 1, 2, 3, 4\} }}$. Generally, larger $\alpha$ leads to better clustering performance. For example, we let $\lambda$ range from 0.15 to 0.4 with $\alpha=3$. Then, the clustering error varies from $3.91\%$ to $6.41\%$ (Fig. \ref{fig:face:a}). When we let $\lambda$ range from 0.15 to 0.4 with $\alpha=1$, the clustering error varies from $15.31\%$ to $42.19\%$. However, note that if $\alpha$ is too large (i.e., $\alpha=4$). LRRSC must narrow the range of parameter $\lambda$ to obtain the desired result. This can also be observed in Fig. Fig. \ref{fig:face:b}. Fig. \ref{fig:face:c} shows that LRRSC performs well for a large range of $\lambda$. This is benefited from the 100-dimensional data obtained by applying PCA with noise removal. Consequently, LRRSC is capable of a stable face clustering performance when $\lambda$ is chosen according to the noise level.

\begin{table}[!htbp]
\small
\setlength{\abovecaptionskip}{0pt}
\setlength{\belowcaptionskip}{10pt}
\setlength{\tabcolsep}{1pt}
\centering
\caption{Clustering error (\%) of different algorithms on the first 10 classes of the Extended Yale B Database.}
\label{tb:face1}
\begin{tabular}{|cccccccc|}
\hline
Algorithm & LRRSC & eLRRSC & SLRR & LRR & LRR-PSD & LRSC &  SSC\\
\hline
Dim. = 100  &  4.37 & \textbf{4.22} & \textbf{4.22} & 21.09 & 38.44 & 36.47 & 36.56 \\
Dim. = 300 & \textbf{3.91} & \textbf{3.91} & \textbf{3.91} & 20.63 & 38.12 & 35.78 & 35.47 \\
Dim. = 500 & 3.91 & \textbf{2.97} & 3.13 & 21.56 & 35.47 & 36.67 & 35 \\
Raw data & 3.91 &  \textbf{2.97} & - & 20.94 & 35.47 & 36.97 & 35 \\
\hline
\end{tabular}
\end{table}

The three different feature dimensions of the face images required 32, 118, 114 and 113 iterations. Table \ref{tb:face1} shows the proposed eLRRSC has most promising performance. For example, eLRRSC achieved a low clustering error of $2.97\%$ for the original data, and improved the clustering accuracy by at least $18\%$ when compared with LRR, LRSC and SSC. We observed the same advantages when using our proposed method for the 500- and 100-dimensional data obtained using PCA. The clustering results for each algorithm using raw or reduced dimension data are very similar, which suggests that the face images of an individual lie close to a union of subspaces. These clustering results confirmed that the affinity calculated from the symmetric low-rank representation significantly improves the clustering accuracy when the data are grossly contaminated by noise, and that it outperforms the other algorithms. LRR compares favorably against the other algorithms. LRR-PSD, SSC, and LRSC have very similar clustering results.

Fig. \ref{fig:face12} showed the computational times of the competing algorithms corresponding to results outside parentheses in Table \ref{tb:face1}. We can see that the computation costs of eLRRSC, SLRR and LRSC significantly outperformed the other approaches. This is because both of them can obtain a closed form solution of the low-rank representation, which they uses to build the affinity. Thus, they can run much faster than the other approaches. On other hand, LRRSC, LRR, LRR-PSD have comparable computational times because of their efficient convex optimization techniques.

\begin{figure}[htbp]
\centering
\includegraphics[width=7cm]{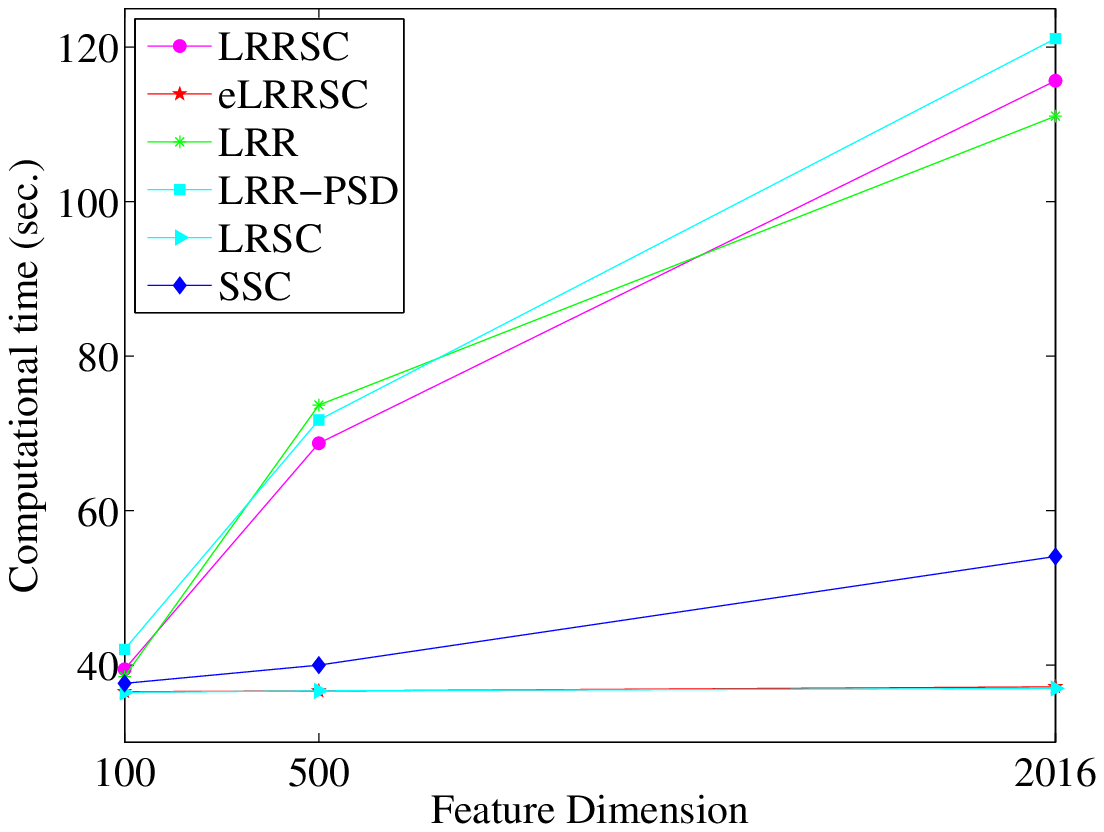}
\caption{Computational time (seconds) of each algorithm for different feature dimensions, using the first 10 classes of the Extended Yale B Database.}
\label{fig:face12}
\end{figure}

Since clustering performance of LRRSC and eLRRSC is closely related with the choice of the parameters, we conducted another experiment to illustrate the effect of estimation of the parameters of LRRSC and eLRRSC using various number of training samples. We designed four groups of training samples, where each group randomly selected 5, 10, 15 and 20 images of each person respectively. The rest are used for testing. Because the training and testing samples are selected randomly, 10 different training and test sample sets are chosen for parameter evaluation. The final clustering result is computed by averaging the recognition rates from these ten experiments. We set parameter $\alpha$ range from 2 to 4. For LRRSC, we let $\lambda$ range from 0.1 to 0.3 in steps of 0.05. For eLRRSC, and we let $\lambda$ range from 5 to 40 in steps of 5, and let $\mu$ range from 0.01 to 0.3 in steps of 0.02 respectively.

Table \ref{tb:parameter} shows the mean clustering error and standard deviation of LRRSC and eLRRSC when the number of an individual's images varies from 5 to 20. LRRSC and eLRRSC obtained similar clustering results under different numbers of training samples. However, eLRRSC achieved a lower computation cost because its solution can be computed in closed form. Besides, it can be seen from Table \ref{tb:parameter} that the mean clustering error gradually raises as the number of samples increases. The larger the number of samples there is in the experiment, the greater the computational complexity there is for clustering regarding as an unsupervised manner.

\begin{table}[!htbp]
\small
\setlength{\abovecaptionskip}{0pt}
\setlength{\belowcaptionskip}{10pt}
\setlength{\tabcolsep}{5pt}
\centering
\caption{Clustering error (\%) and computational time (seconds) of LRRSC and eLRRSC on the first 10 classes of the Extended Yale B Database using different number of training samples for parameter evaluation.}
\label{tb:parameter}
\begin{tabular}{|c|ccc|ccc|}
\hline
\multirow{2}{*}{Algorithm} & \multicolumn{3}{|c|}{LRRSC} & \multicolumn{3}{|c|}{eLRRSC} \\
\cline{2-7}
& Mean & Std. & Time & Mean & Std. & Time \\
\hline
5  & 7.68 & 3.8 & 119.14 & \textbf{6.22} & \textbf{3.29} & \textbf{34.24} \\
10 & \textbf{8.15} & 4.73 & 96.65 & 8.2 & \textbf{4.58} & \textbf{31.27}\\
15 & 9.2 & \textbf{4.44} & 77.82 & \textbf{9.06} & 4.92 & \textbf{29.58}\\
20 & \textbf{9.34} & \textbf{4.42} & 68.53 & 11.34 & 5.11 & \textbf{27.3}\\
\hline
\end{tabular}
\end{table}

Finally, we evaluated the performance and robustness of LRRSC and eLRRSC as well as the other methods on a more challenging set of face images using artificial occlusion, namely random pixel corruptions. To simulate random pixel corruptions, the locations of corrupted pixels of face images were chosen randomly with uniformly distributed random values in the range [0, 1]. The percentage of pixel corruption levels was varied from 10 to 40\% in steps of 10\%. Figure  \ref{fig:yalebb} shows some examples of the face images with random 10\% pixel corruptions.  All experiments were repeated 10 times. Table \ref{tb:corruption} shows the average clustering error. Some experiment results are given by our previous work \cite{Chen2016SLRR}. The results demonstrate that LRRSC, eLRRSC and SLRR obtain similar clustering accuracies. At corruption percentages of 10\% and 20\%, SLRR obtains the lowest clustering error. Besides, LRRSC consistently outperformed all the other methods for larger percentages of corrupted pixels, i.e., corruption percentages of 30\% and 40\%. Compared with the other competing methods, LRRSC, eLRRSC and SLRR are slightly more stable as the percentage of corruption increases.

\begin{table}[!htbp]
\small
\setlength{\abovecaptionskip}{0pt}
\setlength{\belowcaptionskip}{10pt}
\setlength{\tabcolsep}{2pt}
\centering
\caption{Clustering error (\%) by applying different algorithms on the first 10 classes of the Extended Yale Database B contaminated by random pixel corruptions.}
\label{tb:corruption}
\begin{tabular}{|c|ccccccc|}
\hline
Ratio (\%) & LRRSC & eLRRSC & SLRR &  LRR  & LRR-PSD & LRSC & SSC \\
\hline
  10 & 12.16 & 11.28 & \textbf{9.23} & 21.38 & 25.31 & 16.22 & 32.84 \\
\hline
  20 & 12.25 & 11.51 & \textbf{10.34} & 24.77  & 26.01  & 17.47 & 39.44 \\
\hline
  30 & \textbf{12.79} & 13.18 & 13.69 & 30.44  & 30.55 & 17.2 & 43.84 \\
\hline
  40 & \textbf{12.23} & 12.58 & 14.59 & 31.72 & 31.02  & 20.72 & 48.95 \\
\hline
\end{tabular}
\end{table}

\texttt{2. Second experimental scenario: } We used the experimental settings from \cite{Elhamifar2013SSC}. We divided the 38 subjects into four groups, where subjects 1 to 10, 11 to 20, 21 to 30, and 31 to 38 correspond to four different groups. For each of the first three groups, we considered all choices of $n \in \{ 2,3,5,8,10\}$. For the last group, we considered all choices of $n \in \{ 2,3,5,8\} $. We tested each choice (i.e., each set of $n$ subjects) using each algorithm. Finally, the mean and median subspace clustering errors for different number of subjects were computed using all algorithms. Note that we applied these clustering algorithms to the normalized face images.

\begin{table}[!htbp]
\small
\setlength{\abovecaptionskip}{0pt}
\setlength{\belowcaptionskip}{10pt}
\setlength{\tabcolsep}{1pt}
\centering
\caption{Average clustering error (\%) for different number of subjects, applying each algorithm to the Extended Yale B database.}
\label{tb:face2}
\begin{tabular}{|cccccccc|}
\hline
Algorithm & LRRSC & eLRRSC & SLRR & LRR & LRR-PSD & LRSC & SSC \\
\hline
2 Subjects & & & & & & &\\
Mean & 1.78  & 1.32 & \textbf{1.29} & 2.54 & 3.04 & 4.25 & 1.86 \\
Median & 0.78  & 0.78 & 0.78 & 0.78 & 2.34 & 3.13 & \textbf{0} \\
\hline
3 Subjects & & & & & & &\\
Mean & 2.61  & 2.08 & \textbf{1.94} & 4.23 & 4.33 & 6.07& 3.24 \\
Median & 1.56 & 1.56 & 1.56  & 2.6 & 3.91 & 5.73 & \textbf{1.04} \\
\hline
5 Subjects & & & & & & &\\
Mean & 3.19  & \textbf{2.5} & 2.72 & 6.92 & 10.45 & 10.19 & 4.33 \\
Median & 2.81  & \textbf{2.19} & 2.5 & 5.63 & 7.19 & 7.5 & 2.82 \\
\hline
8 Subjects & & & & & & &\\
Mean & 4.01  & \textbf{3.02} & 3.21 & 13.62 & 23.86 &  23.65 & 5.87 \\
Median & 3.13 & \textbf{2.34} & 2.93 & 9.67 & 28.61 & 27.83 & 4.49 \\
\hline
10 Subjects & & & & & & &\\
Mean & 3.7 & \textbf{3.28} & 3.49 & 14.58 & 32.55 & 31.46 & 7.29 \\
Median & 3.28 & \textbf{2.81} & \textbf{2.81} & 16.56 & 34.06 & 28.13 & 5.47 \\
\hline
\end{tabular}
\end{table}

Table \ref{tb:face2} shows the clustering results of various approaches using different number of subjects. When considering two subjects, eLRRSC achieved a clustering error of $1.32\%$. The eLRRSC algorithm consistently obtained lower average clustering errors than the other algorithms when the number of subjects increased. For example, there was nearly $4\%$ improvement in clustering accuracy compared with SSC for 10 subjects. Note that SSC performed better than the other original LRR-based approaches (LRR, LRR-PSD, and LRSC) in terms of the average clustering error. This confirms that our proposed method is very effective and robust to different number of subjects for face clustering.

\begin{figure*}[!htbp]
\begin{minipage}[t]{0.14\linewidth}
\centering
\subfigure[LRRSC]{
\label{fig:affinity:a}
\includegraphics[width=1\textwidth]{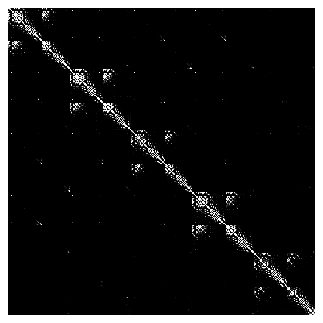}}
\end{minipage}%
\begin{minipage}[t]{0.14\linewidth}
\centering
\subfigure[eLRRSC]{
\label{fig:affinity:b}
\includegraphics[width=1\textwidth]{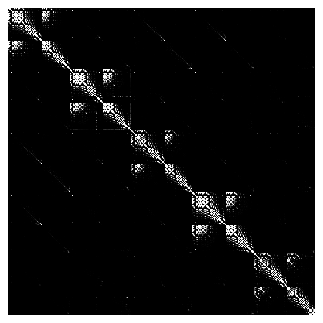}}
\end{minipage}
\begin{minipage}[t]{0.14\linewidth}
\centering
\subfigure[SLRR]{
\label{fig:affinity:d}
\includegraphics[width=1\textwidth]{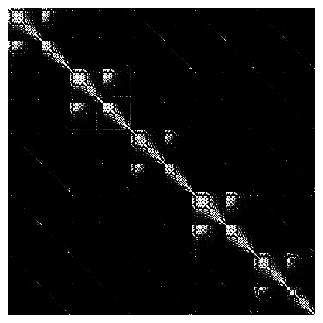}}
\end{minipage}%
\begin{minipage}[t]{0.14\linewidth}
\centering
\subfigure[LRR]{
\label{fig:affinity:e}
\includegraphics[width=1\textwidth]{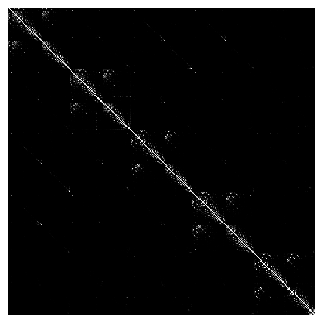}}
\end{minipage}
\begin{minipage}[t]{0.14\linewidth}
\centering
\subfigure[LRR-PSD]{
\label{fig:affinity:f}
\includegraphics[width=1\textwidth]{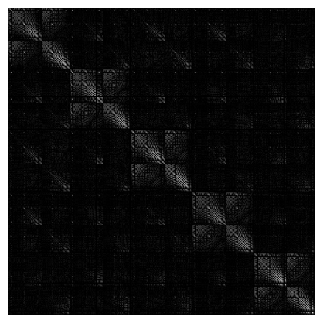}}
\end{minipage}
\begin{minipage}[t]{0.14\linewidth}
\centering
\subfigure[LRSC]{
\label{fig:affinity:g}
\includegraphics[width=1\textwidth]{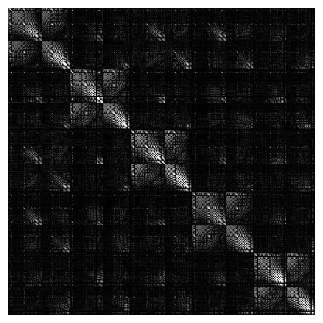}}
\end{minipage}
\begin{minipage}[t]{0.14\linewidth}
\centering
\subfigure[SSC]{
\label{fig:affinity:h}
\includegraphics[width=1\textwidth]{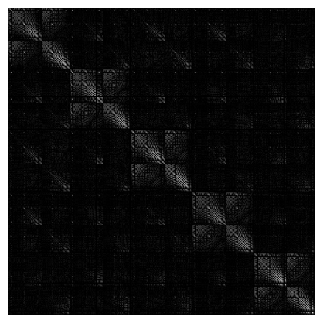}}
\end{minipage}
\caption{Examples of the affinity graph matrix produced when using different algorithms for the Extended Yale B Database with five subjects.}
\label{fig:affinity}
\end{figure*}

Fig. \ref{fig:affinity} shows examples of the affinity graph matrix produced by the different algorithms for the Extended Yale B Database with five subjects. It is clear from Fig. \ref{fig:affinity} that the affinity graph matrix produced by LRRSC and eLRRSC has a distinct block-diagonal structure, whereas the other approaches do not. The clear block-diagonal structure implies that each subject becomes highly compact and the different subjects are better separated. Note that the number of iterations taken by our algorithm is always less than 335. The average numbers of iterations for the five different sets ($n \in \{ 2,3,5,8,10\} $ subjects) in the proposed method are 308, 256, 202, 182 and 166.

Fig. \ref{fig:face13} shows the mean computational times for the Extended Yale B database with different number of subjects. Note that the computational times of eLRRSC, SLRR and LRSC are still much lower than that of the other algorithms. However, LRSC does not perform well, especially as the number of subjects increases. LRRSC, LRR, and SSC have comparable computational times in these experiments.

\begin{figure}[htbp]
\centering
\includegraphics[width=9cm]{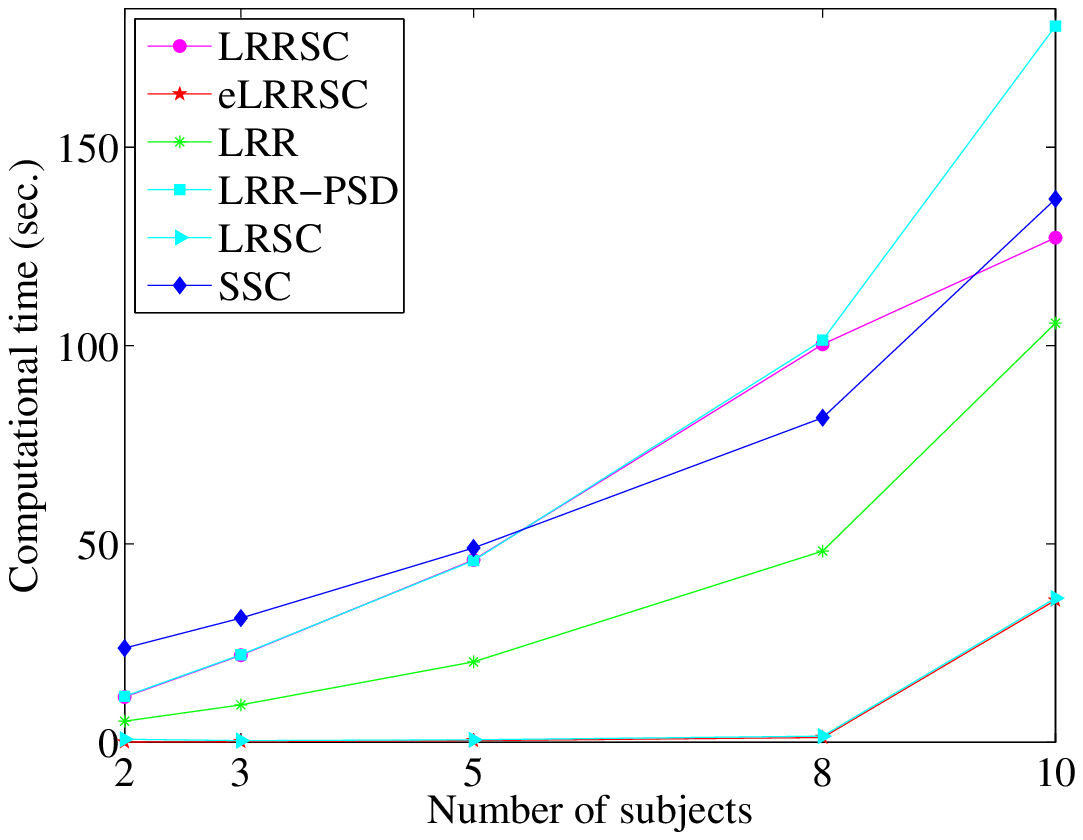}
\caption{Comparison of mean computational time (seconds) of each algorithm applied to the Extended Yale B database, using different number of subjects.}
\label{fig:face13}
\end{figure}

\subsection{Experiments on motion segmentation}

Motion segmentation refers to the problem of segmenting feature trajectories of multiple rigidly moving objects into their corresponding spatiotemporal regions in a video sequence. The feature trajectories from a single rigid motion lie in a linear subspace of at most four dimensions \cite{Boult1991Factor}. Motion segmentation can be regarded as a subspace clustering problem.

\subsubsection{The Hopkins 155 database}

We first consider the Hopkins 155 database \cite{Hopkins155} for the motion segmentation problem. It consists of 155 sequences of two motions or three motions. Fig. \ref{fig:h155example} shows some example frames from two video sequences with feature trajectories. There are $39-550$ data vectors drawn from two, three or five motions for each sequence, which correspond to a subspace. Each sequence is a separate clustering task.

\begin{figure}[!htbp]
\begin{minipage}[t]{0.5\linewidth}
\centering
\subfigure[]{
\label{fig:h155:a} 
\includegraphics[width=4cm]{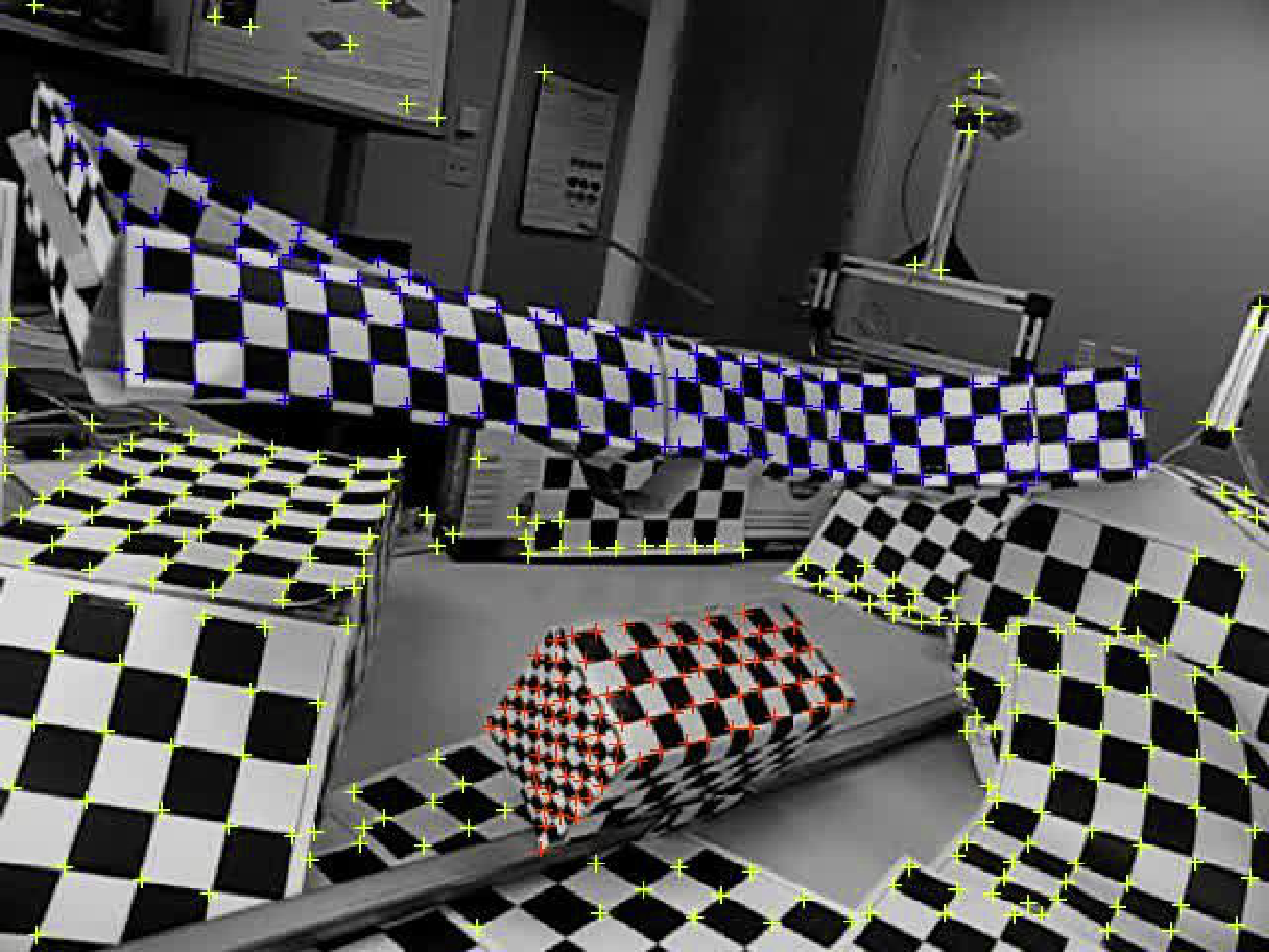}}
\end{minipage}%
\begin{minipage}[t]{0.5\linewidth}
\centering
\subfigure[]{
\label{fig:h155:b} 
\includegraphics[width=4cm]{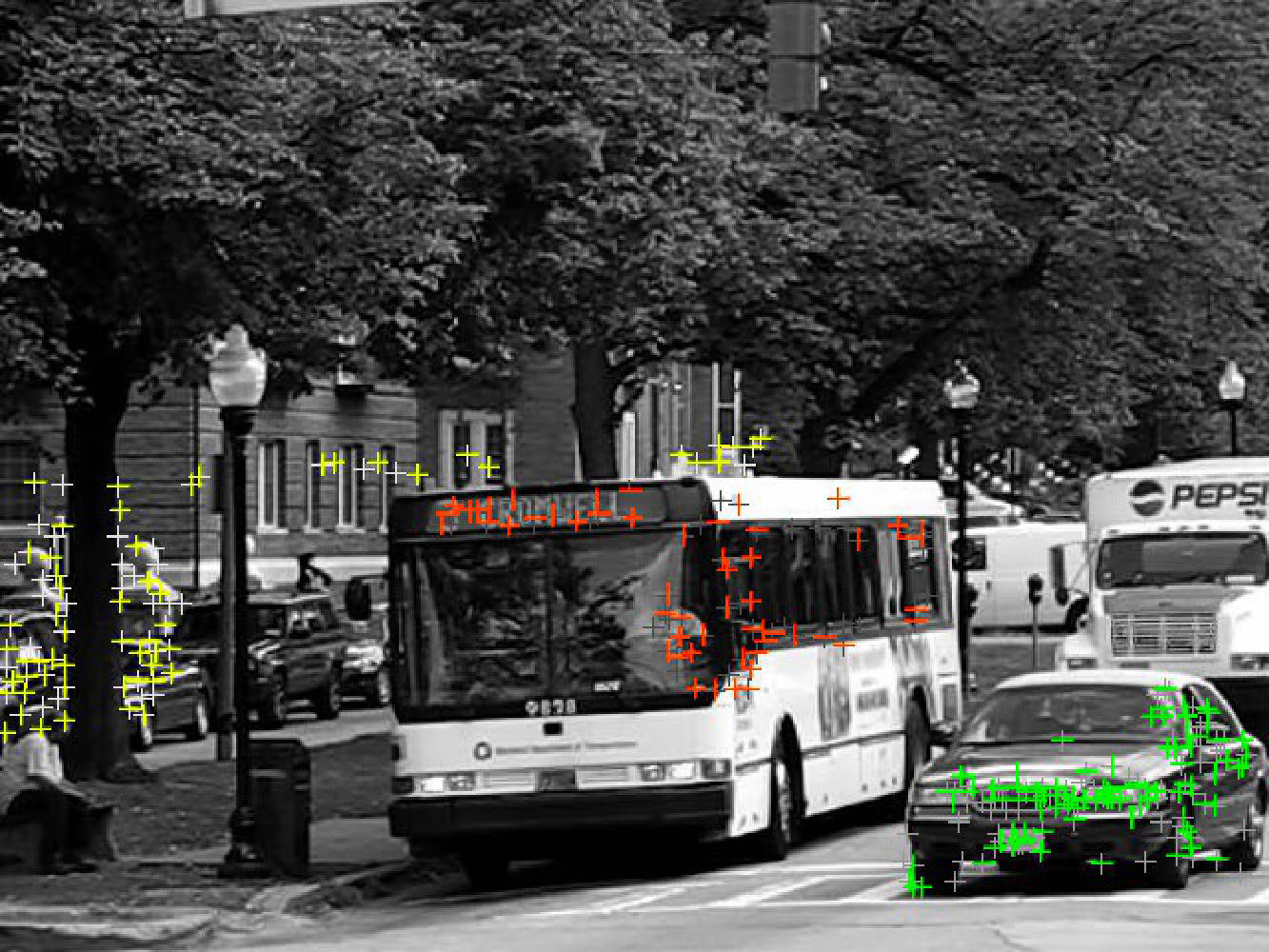}}
\end{minipage}
\caption {Example frames from two video sequences with feature trajectories involving three motions.}
\label{fig:h155example} 
\end{figure}

We considered two scenarios to evaluate the performance of the applicability of LRRSC and eLRRSC to motion segmentation. We first used the original feature trajectories associated with each motion in an affine subspace, i.e., enforced that the coefficients sum to $1$. Then, we projected the original data into a $4n$-dimensional subspace using PCA, where $n$ is the number of subspaces.

The clustering performance for all 155 sequences was largely affected by $\lambda$. Fig. \ref{fig:h155:a} and \ref{fig:h155:b} show the clustering errors of LRRSC using two experimental settings for different $\lambda$ over all 155 sequences. In Fig. \ref{fig:h155:a}, when $\lambda$ was between 1 and 6 the clustering error varied between $1.5\%$ and $2.67\%$; when $\lambda$ was between 2.4 and 4, the clustering error remained very stable, varying between $1.5\%$ and $1.7\%$. In Fig. \ref{fig:h155:b}, when $\lambda$ was between 1 and 6, the clustering error varied from $1.56\%$ to $3.18\%$; when $\lambda$ was between 2.4 and 4, the clustering error varied from $1.56\%$ to $2.31\%$ and remained very stable. This implies that LRRSC performs well under a wide range of values of $\lambda$. In addition, choosing $\lambda$ for each sequence may improve the clustering performance, especially when the data are grossly corrupted by noise.

\begin{figure}[!htbp]
\begin{minipage}[t]{0.5\linewidth}
\centering
\subfigure[]{
\label{fig:h155:a} 
\includegraphics[width=4.4cm]{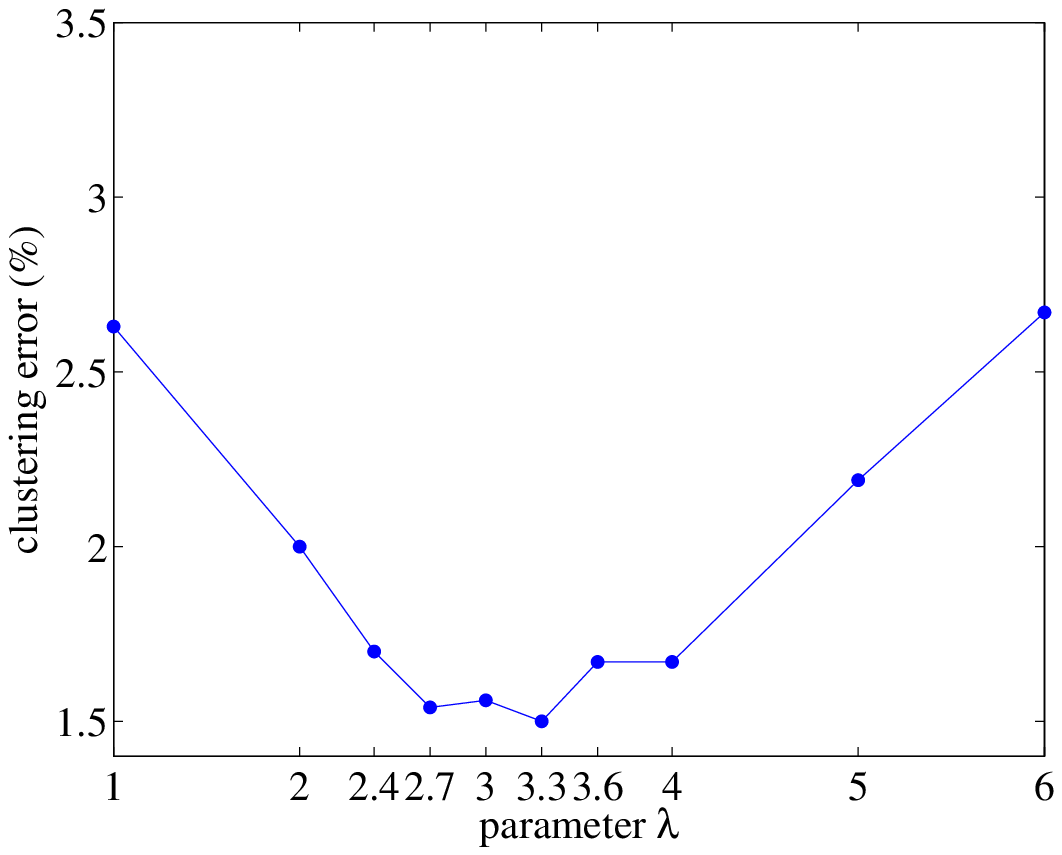}}
\end{minipage}%
\begin{minipage}[t]{0.5\linewidth}
\centering
\subfigure[]{
\label{fig:h155:b} 
\includegraphics[width=4.4cm]{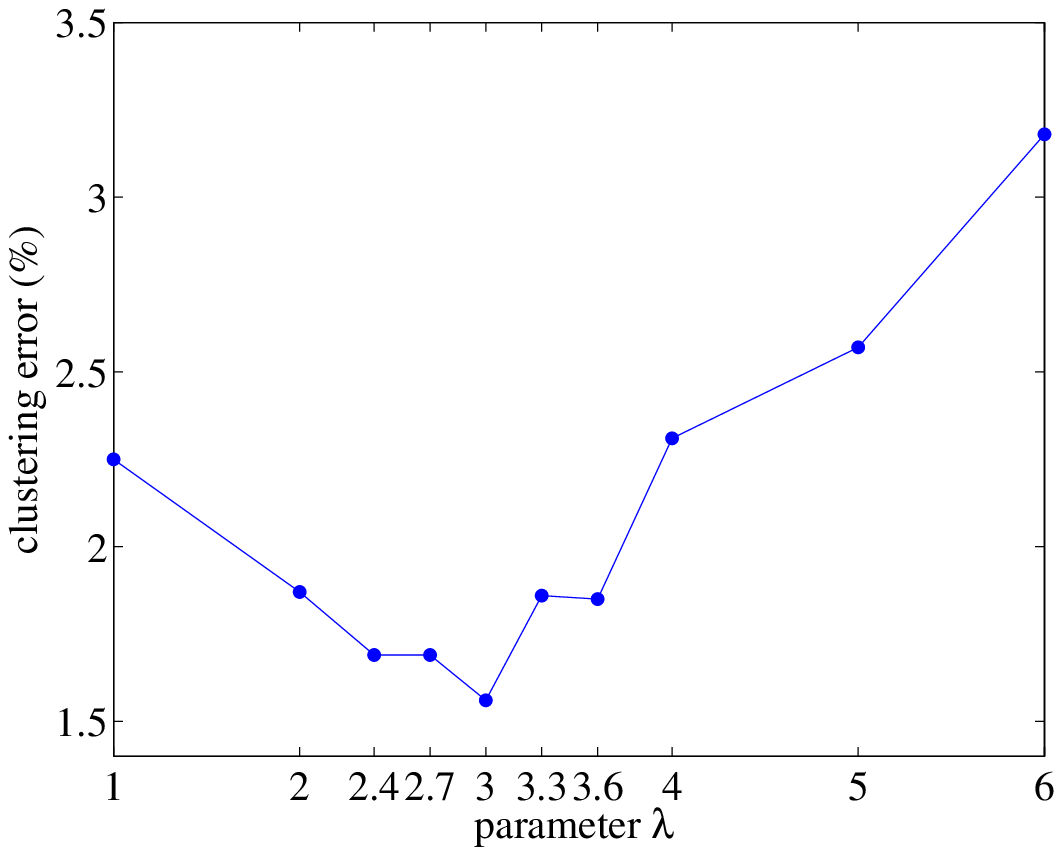}}
\end{minipage}
\caption {The influences of the parameter $\lambda$ of LRRSC. (a) The clustering error of LRRSC on the Hopkins 155 database with the $2F$-dimensional data points. (b) The clustering error of LRRSC on the Hopkins 155 database with the 4n-dimensional data points obtained by applying PCA. }
\label{fig:h155} 
\end{figure}

\begin{table}[!htbp]
\small
\setlength{\abovecaptionskip}{0pt}
\setlength{\belowcaptionskip}{10pt}
\setlength{\tabcolsep}{5pt}
\centering
\caption{Average clustering error (\%) and mean computation time (seconds) when applying the different algorithms to the Honkins 155 database, with the $2F$-dimensional data points.}
\label{h155tb1}
\begin{tabular}{|c|c|c|c|c|c|}
\hline
\multirow{2}{*}{Algorithm} & \multicolumn{4}{|c|}{Error} & \multirow{2}{*}{Time} \\
\cline{2-5}
 & mean & median & std. & max. &\\
\hline
  LRRSC & 1.5  & \textbf{0}  & 4.36 & \textbf{33.33} & 4.71  \\
  eLRRSC & \textbf{0.86} & \textbf{0} & \textbf{3.39} & 34.33  & \textbf{0.09} \\
 SLRR & 0.88  & \textbf{0}  & 3.63 & 38.06 & \textbf{0.09} \\
  LRR & 1.71  & 0 & 4.86   & 33.33 & 1.29 \\
  LRR-PSD  & 5.38  & 0.55  & 10.18 & 45.79   & 4.35 \\
  LRSC & 4.73  & 0.59  & 8.8 & 40.55 & 0.14 \\
  SSC & 2.23  & \textbf{0}  & 7.26  & 47.19 & 1.02 \\
\hline
\end{tabular}
\end{table}

\begin{table}[!htbp]
\small
\setlength{\abovecaptionskip}{0pt}
\setlength{\belowcaptionskip}{10pt}
\setlength{\tabcolsep}{3pt}
\centering
\caption{Average clustering error (\%) and mean computation time (seconds) when applying the different algorithms to the Honkins 155 database, with the $4n$-dimensional data points obtained using PCA.}
\label{h155tb2}
\begin{tabular}{|c|c|c|c|c|c|}
\hline
\multirow{2}{*}{Algorithm} & \multicolumn{4}{|c|}{Error} & \multirow{2}{*}{Time} \\
\cline{2-5}
 & mean & median & std. & max. &\\
\hline
  LRRSC & 1.56 & \textbf{0} & 5.48 & 43.38 & 4.62 \\
  eLRRSC  & \textbf{1.3}  & \textbf{0} & \textbf{4.97} & \textbf{39.73} & 0.08 \\
  SLRR & \textbf{1.3}& \textbf{0}  & 5.1 & 42.16 & \textbf{0.07} \\
  LRR & 2.17  & \textbf{0} & 6.58 & 43.38 & 0.69 \\
  LRR-PSD & 5.78  & 0.57 & 10.6 & 45.79 & 5.23 \\
  LRSC & 4.89  & 0.63  & 8.91  & 40.55 & 0.13 \\
  SSC & 2.47  & \textbf{0}  & 7.5  & 47.19 & 0.93 \\
\hline
\end{tabular}
\end{table}

Table \ref{h155tb1} and \ref{h155tb2} show the average clustering errors of the different algorithms on all 155 sequences of the Hopkins 155 database, using two experimental settings. In both experimental settings, eLRRSC obtained competitive clustering results and significantly outperformed the other algorithms. Specifically, LRRSC obtained $1.5\%$ and $1.56\%$, and eLRRSC obtained $0.86\%$ and $1.3\%$ clustering errors for the two experimental settings. This confirms the effectiveness and robustness of LRRSC and eLRRSC for the segmentation of different motion subspaces by exploiting the angular information of the principal directions of the symmetric low-rank representation of each motion. We note that the accuracy of LRRSC was very similar for the two experimental settings. This confirms that the feature trajectories of each sequence in a video approximately lie close to a $4n$-dimensional linear subspace of the $2F$-dimensional ambient space \cite{Boult1991Factor}. The computational costs of eLRRSC, SLRR and LRSC are much lower than the other algorithms. Hence, LRRSC and eLRRSC are effective and robust methods for motion segmentation. The computational cost of LRR is lower than the LRR-based methods (LRRSC and LRR-PSD). This is because LRR applies the dictionary learning method to improve performance, while LRRSC uses one SVD computation at each iteration and the original data's dictionary.

\subsubsection{The penguin motion database}

\begin{figure}[!htbp]
\begin{minipage}[t]{0.5\linewidth}
\centering
\subfigure[A original frame]{
\label{fig:penguin:a} 
\includegraphics[width=4cm]{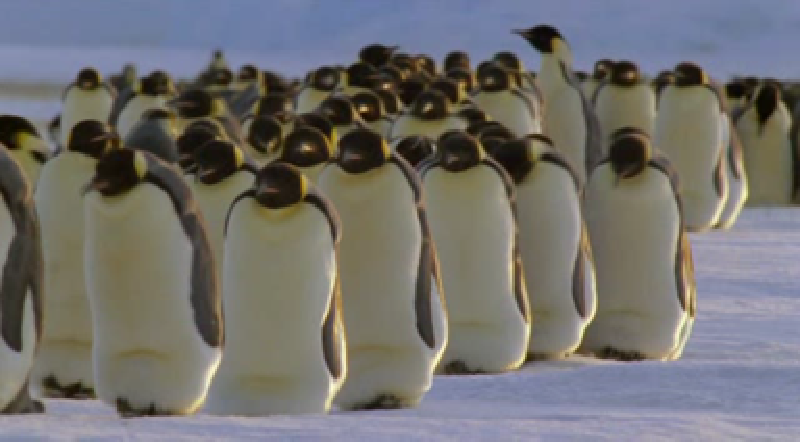}}
\end{minipage}%
\begin{minipage}[t]{0.5\linewidth}
\centering
\subfigure[The annotated frame]{
\label{fig:penguin:b} 
\includegraphics[width=4cm]{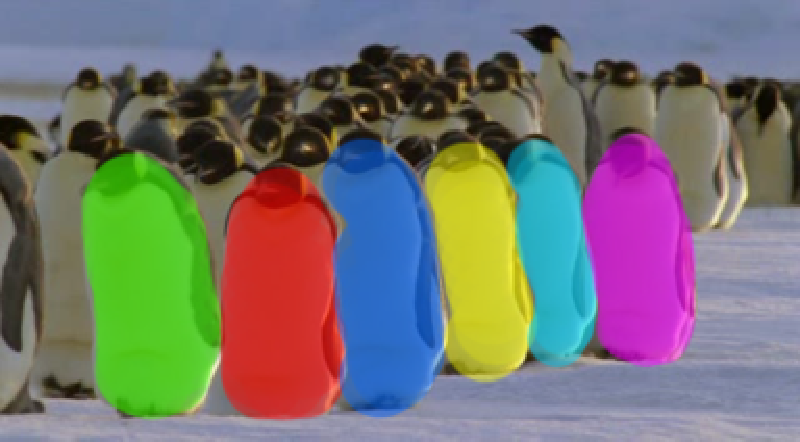}}
\end{minipage}
\caption {A example frame of a video sequence involving six motions.}
\label{fig:penguin:example} 
\end{figure}
In this section, we further compared the clustering performance obtained by LRRSC and eLRRSC against the other competing algorithms on a more challenging data set. The penguin motion database contains six non-rigidly moving objects over 42 annotated frames of a video sequence, which is drawn from the SegTrack dataset \cite{Li2013VS}. We presented four scenarios consisting of four data sets of different feature trajectories to illustrate the robustness and effectiveness of LRRSC and eLRRSC. In particular, we sampled 50, 100, 150 and 200 points of feature trajectories at the average intervals from each annotated frame of the penguin motion database respectively. This experiment is more challenging for all the algorithms. A typical example frame of a video sequence and its annotated frame are illustrated in Fig. \ref{fig:penguin:example}.

\begin{figure}[!htbp]
\centering
\includegraphics[width=8cm]{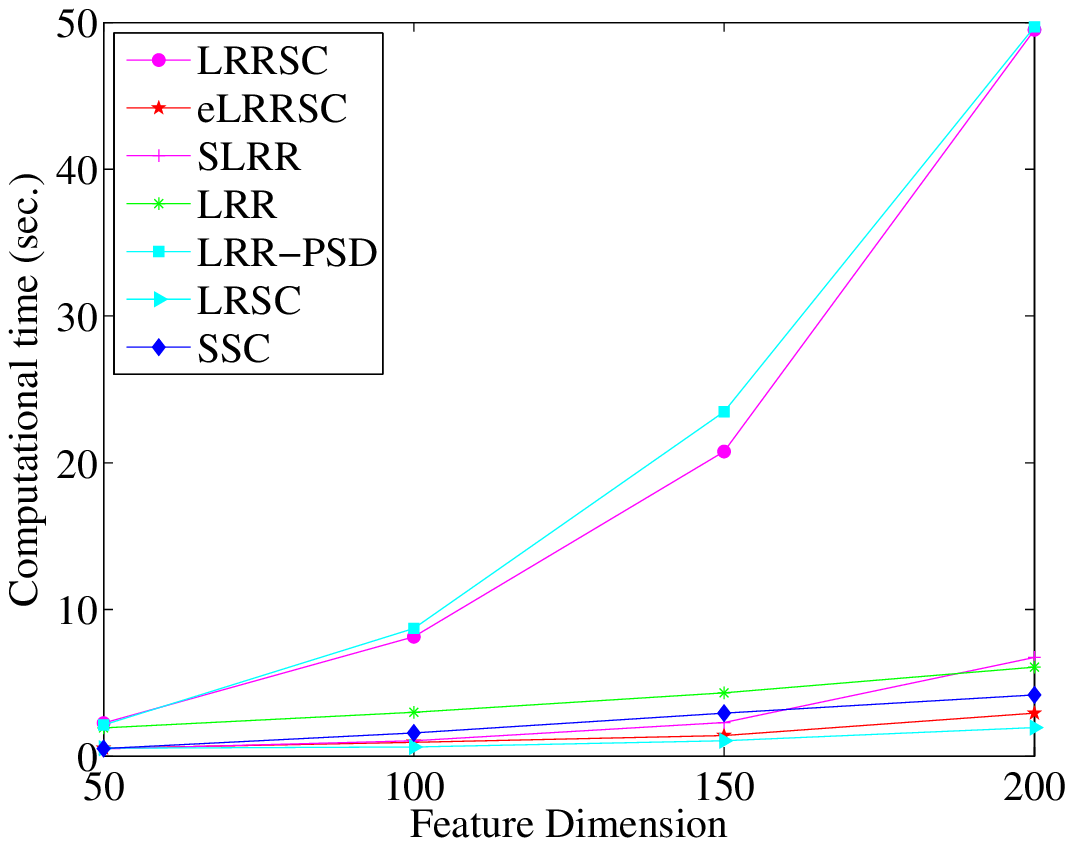}
\caption{Comparison of computational time (seconds) of each algorithm applied to the penguin motion database, using different number of feature trajectories.}
\label{fig:penguin:time}
\end{figure}

We reported clustering errors for four scenarios in Table \ref{tb:penguin}, which demonstrates that our proposed eLRRSC achieves a consistently high clustering accuracy. Besides, LRRSC also performs well in four scenarios. This suggests that the angular information of the principal directions of the symmetric low-rank representation have the potential to make non-rigidly moving objects of a video sequence separable when the number of motions increases. All the results show that SSC always performs poorly under different feature trajectories. Fig. \ref{fig:penguin:time} shows the computational costs of all competing methods. Clearly, eLRRSC, SLRR and LRSC achieve similar computational costs overall, which are lower than that of the other methods.

\begin{table}[!htbp]
\small
\setlength{\abovecaptionskip}{0pt}
\setlength{\belowcaptionskip}{10pt}
\setlength{\tabcolsep}{1pt}
\centering
\caption{Clustering error (\%) of different algorithms on the penguin motion database with different number of feature trajectories.}
\label{tb:penguin}
\begin{tabular}{|c|c|c|c|c|c|c|c|}
\hline
Algorithm & LRRSC & eLRRSC & SLRR &LRR & LRR-PSD & LRSC & SSC \\
\hline
50  & 4 & \textbf{3.33} & 20 & 27.67 & 22 & 21 & 20.33 \\
100 & 7.33 & \textbf{3.17}& 18.67 & 26.5 & 22.33 & 20.5 & 23.5 \\
150 & 8.22 & \textbf{3.78} & 17.67 & 21 & 22.33 & 22.11 & 24.56 \\
200 & 8  & \textbf{3.42} & 18.09 & 21.83 & 28.5 & 21.67 & 20.5 \\
\hline
\end{tabular}
\end{table}

\section{Discussions}
\label{sec:Discussions}

LRR-based techniques such as LRR and LRRSC compared with sparsity based methods able to capture globally linear structures of data. However, there are still several significant problems. For LRRSC and eLRRSC, how to choose proper parameters is an open problem in practice. Although the parameter ${\alpha}$ of LRRSC can be easily chosen empirically by its limited range, it is difficult to estimate the parameter ${\lambda}$ without prior knowledge. In addition, it is important to develop dictionary learning algorithms, which may significantly improve the subspace clustering performance.

Then, we discussed three differences among the above LRR-based methods. First, they contain different objective functions. For example, LRRSC and LRP-PSD imposed different constraints on low-rank representation, i.e., a symmetric constraint and semi-definite guarantees, respectively. In particular, LRRSC only aimed to obtain a symmetric matrix while LRR-PSD pursued a semi-definite matrix at the first step of the optimizations. In fact, we should emphasize that it is easy to validate that their corresponding efficiency, i.e., Lemma \ref{lemma1} and Theorem 14 \cite{Ni2010LRRPSD}, are distinct using synthetic data or some instances. For instance, different optimal solutions at the second step of the eLRRSC algorithm with relative large $\lambda$ can be achieved if we used semi-definite guarantees instead of symmetric constraint with respective optimization theory. Moreover, the procedures of the proof of two optimization theories are different. However, both of Lemma \ref{lemma1} and Theorem 14 can achieve the identical result if and only if $Q$ is a symmetric positive semi-definite matrix in Lemma \ref{lemma1}. This is because the results of SVD and eigenvalue decomposition of the matrix are identical if a symmetric positive semi-definite matrix is given.

Compared with LRR, LRR-PSD and LRSC, the second difference is LRRSC and eLRRSC utilize the angular information of its principal directions of symmetric low-rank representation to build similarity graph under the unified subspace clustering framework. The utilization of the angular information in the proposed algorithm guarantees its robustness towards evaluating the memberships between each pair of data points. Experimental results further demonstrated that it will lead to a significant improvement on the clustering performance. Overall, we can argue that combination of the two improvements plays an important role in the low-rank representation and achieves satisfied results in the subspace clustering.

The last difference is how to decrease the computation cost among the competing methods. The computational complexity of some existing LRR-based methods which require iterative SVD operations becomes computationally impracticable in large-scale subspace clustering problems. Hence, pursing a closed form solution is a positive way to avoid iterative SVD operations. Frobenius-norm plays a critical role in eLRRSC. By making use of the Frobenius norm, eLRRSC obtained a collaborative representation of high-dimensional data. Then the nuclear norm is employed in eLRRSC as a common surrogate for low-rank criterion. Finally, eLRRSC pursues a symmetric low-rank matrix preserving the low-dimensional subspace structures from the collaborative representation in a closed form solution. The experimental results illustrated that the computation costs of eLRRSC, SLRR and LRSC are much lower than that of other algorithms.

\section{Conclusions}
\label{sec:Conclusions}
In this paper, we proposed a method that used a low-rank representation with a symmetric constraint to solve the problem of subspace clustering, using the assumption that high-dimensional data are approximately drawn from a union of multiple subspaces. In contrast with existing low-rank based algorithms, LRRSC integrates the symmetric constraint into the low-rankness property of high-dimensional data representation, which can be efficiently calculated by solving a convex optimization problem. The affinity matrix for spectral clustering can be obtained by further exploiting the angular information of the principal directions of the symmetric low-rank representation. This is a critical step towards understanding the memberships between high-dimensional data points. To speed up the optimization procedures of LRRSC, we also developed the efficient variant of LRRSC (eLRRSC), which considers a closed form solution. Extensive experiments on benchmark databases showed that LRRSC and its variant produce very competitive results for subspace clustering compared with several state-of-the-art subspace clustering algorithms, and demonstrated its robustness when handling noisy real-world data.
\section*{Acknowledgments}

The authors thank the anonymous reviewers for their thorough and valuable comments and suggestions.





\bibliographystyle{model1-num-names}
\bibliography{lrrsc}







\end{document}